\documentclass{article}
\PassOptionsToPackage{numbers,sort&compress}{natbib}
\usepackage[preprint]{neurips_2026}

\usepackage[utf8]{inputenc}
\usepackage[T1]{fontenc}
\usepackage{hyperref}
\usepackage{url}
\usepackage{booktabs}
\usepackage{amsfonts}
\usepackage{amsmath,amssymb,amsthm}
\usepackage{mathtools}
\usepackage{nicefrac}
\usepackage{microtype}
\usepackage{xcolor}
\usepackage{multirow}
\usepackage{makecell}
\usepackage{enumerate}
\usepackage{enumitem}

\theoremstyle{plain}
\newtheorem{theorem}{Theorem}[section]

\theoremstyle{remark}

\usepackage{subfiles}
\usepackage{graphicx}

\usepackage{algorithm}
\usepackage{algpseudocode}
\usepackage{caption}
\newcommand\blfootnote[1]{%
  \begingroup
  \renewcommand\thefootnote{}\footnote{#1}%
  \addtocounter{footnote}{-1}%
  \endgroup
}

\usepackage{amsthm}
\newtheorem{assumption}{Assumption}[section]

\title{Divergence Decoding: Training-Free Capability Fusion}

\author{
Yimi Wang$^{1,\diamond}$ \quad
Hao Li$^{1,\diamond}$ \quad
Shuo Yang$^{1,\diamond}$ \quad
He Cao$^{2}$ \quad
Dechen Zhang$^{3}$ 
\\
Ziang Wu$^{1}$ \quad
Zhiyuan Yan$^{1}$ \quad
Fanyang Mo$^{1,\dagger}$ \quad
Li Yuan$^{1,\dagger}$
\\[1em]
\small
$^{1}$ Peking University
\quad
$^{2}$ International Digital Economy Academy (IDEA)
\quad
$^{3}$ The University of Hong Kong
\\
\small
yimiwang25, shuo\_yang@stu.pku.edu.cn, lihao1984, yuanli-ece@pku.edu.cn
}

\begin{document}

\maketitle

\begin{abstract}
  While large language models excel in reasoning, these generalists often lack knowledge for specialized scientific domains. Conversely, domain models~(specialists), while knowledgeable, suffer from specialization side-effects including diminished logic and reduced robustness.
  To address this dilemma, we introduce \textbf{\textit{Divergence Decoding}}, a training-free framework for capability fusion. It reconstructs the "draft-and-verify" skeleton of speculative decoding into an adaptive routing mechanism. The core is using Jensen-Shannon divergence to monitor the distributional disagreement between the two models at each token. When the specialist exhibits significant divergence, our method identifies it as a potential reasoning risk and instantaneously routes control to the generalist. This allows the dynamic injection of general reasoning while preserving domain expertise, achieving inference-time policy composition of the generalist and the specialist. 
  We evaluate \textit{Divergence Decoding} across diverse model families (Qwen and Llama series) on challenging scientific benchmarks (GPQA, ChemBench, and ChemCoTBench). Experimental results demonstrate that \textit{Divergence Decoding} outperforms both the domain-specialized and general-purpose models, effectively surpassing the performance of most single-model baseline. This suggests that \textit{Divergence Decoding} provides a general, training-free paradigm for fusing diverse LLM capabilities through adaptive inference-time collaboration. \blfootnote{$^\diamond$ Equal contributors, $^\dagger$ Corresponding Authors}

\end{abstract}

\section{Introduction}
The development of large language models (LLMs) has broadly followed two directions. General-purpose LLMs~\cite{bai2023qwen, qwen2024qwen2, yang2025qwen3, qwen3.5, touvron2023llama, grattafiori2024llama, touvron2023llama, jiang2023mistral, guo2025deepseek, Ouyang2022instructgpt, openai2023gpt4, claude2, claude3, geminiteam2023gemini, gemma} emphasize strong reasoning capabilities, often strengthened through large-scale reinforcement learning, and have demonstrated impressive performance on multi-step problem solving and logical consistency across diverse domains. In parallel, scientific-domain specialist LLMs focus on incorporating expert knowledge through continued pretraining or fine-tuning on discipline-specific corpora in chemistry, biology, and materials science~\cite{zhao2024chemdfm, zhao2025chemdfmr, li2026agentic, labrak2024biomistral,bai2025intern, zhang2024comprehensive, hu2025survey}. While these models excel at scientific terminology and specialized tasks, they often remain less robust than general-purpose models in broad reasoning and out-of-domain generalization.
This dichotomy creates a critical bottleneck in scientific reasoning, where a single response necessitates both expert-level knowledge and rigorous multi-step logic. Recent scientific benchmarks~\cite{rein2023gpqa,phan2025humanity,li2025beyond,wang2024mmlu,hendrycks2020measuring,auer2023sciqa,lu2024moleculeqa} highlight that even frontier models struggle to navigate this boundary, often failing either at the retrieval of domain knowledge or the maintenance of reasoning consistency. The complementary nature of these failure modes raises a question: \textit{Can we combine these disparate strengths at inference time to outperform the capabilities of individual models?}

Existing paradigms for multi-model collaboration fall short of this goal. \textit{Speculative decoding}~\cite{leviathan2023fastspeculative}, while structurally similar to our work, is strictly constrained by distributional invariance, aiming solely for acceleration. Conversely, model routing and cascades operate at the coarse query level, failing to exploit the dynamic, token-level interplay between expertise and reasoning. To bridge this gap, we propose \textbf{\textit{Divergence Decoding}}, a training-free framework for capability fusion between a science-specialized LLM and a general reasoning LLM. The specialized model serves as the default generator, while the general model monitors local predictive disagreement through the Jensen-Shannon~(JS) divergence between their next-token distributions. When the two models are locally consistent, we preserve the specialized model's token; when they diverge beyond a threshold, we reject it and fall back to the general reasoning model. In this way, the decoder adaptively integrates domain expertise and general reasoning at the token level, without training an additional router, jointly finetuning the models, or committing to either model globally.

\begin{figure}[t]
    \centering
    \includegraphics[width=1.0\linewidth]{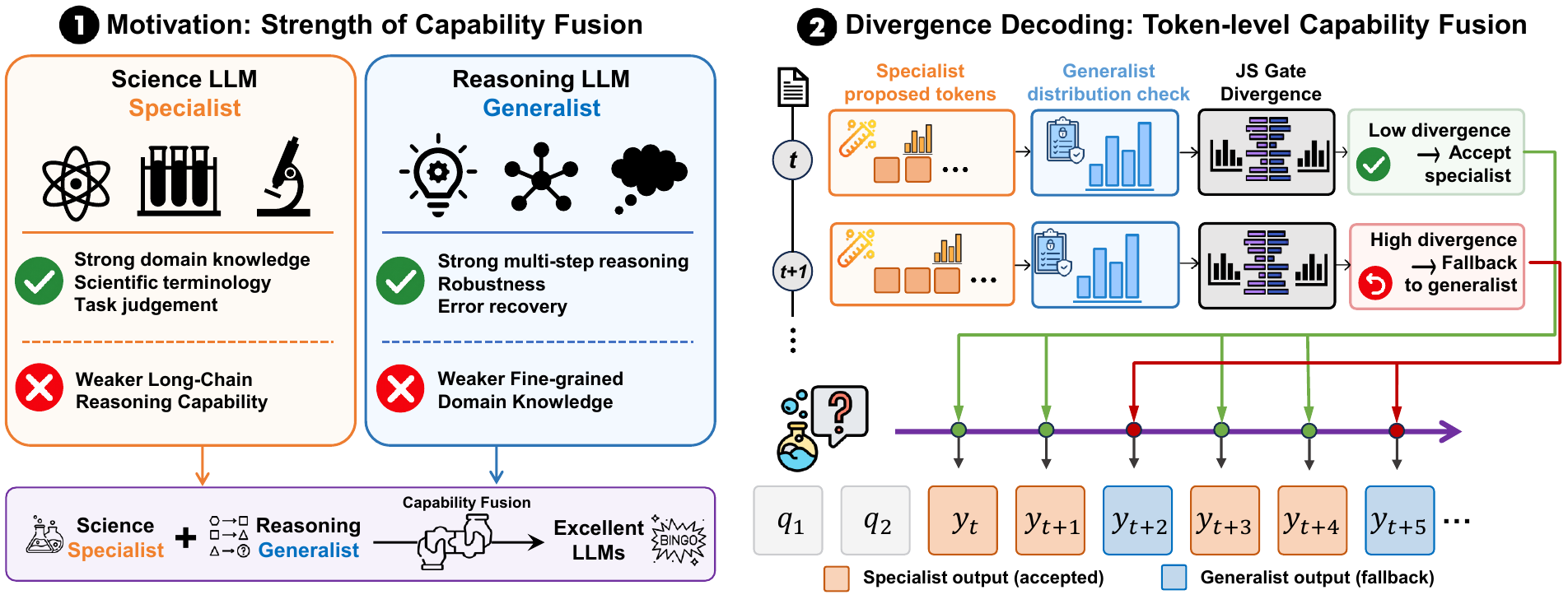}
    \vspace{-0.2in}
    \caption{
    \textbf{Illustration of the proposed Divergence Decoding framework.} Left: Scientific tasks often demand both fine-grained domain knowledge and long-chain reasoning, which are typically bifurcated between specialized and generalist LLMs. Right: Our method bridges this gap through a state-dependent routing policy. By measuring the distribution divergence between the two models at each step $t$, the decoder dynamically navigates between the \textcolor{orange!90!black}{specialist}’s expertise and the \textcolor{blue}{generalist}’s consistency without additional training or joint fine-tuning.
    }
    \vspace{-0.2in}
    \label{fig:placeholder}
\end{figure}

A key point is that our method should not be interpreted as a variant of distribution-preserving speculative decoding. The goal of speculative decoding is to reproduce the target model's distribution more efficiently~\cite{leviathan2023fastspeculative, chen2023accelerating}. Our goal is different: we seek a state-dependent routing policy whose output distribution is neither LLMs, but an adaptive composition of the two. Accordingly, the central theoretical question is not whether the combined trajectory matches any backbone model, but why such a routing policy can outperform either individual model. We answer this question through a theoretical analysis showing that JS divergence is not merely a heuristic disagreement score: under a structural view of complementary model errors, it provides a theoretically justified routing criterion for when the specialist should remain in control and when the generalist should intervene. In particular, low-divergence states can certify regimes where the specialist's domain-specific correction improves over the general model, while high-divergence states identify regimes where falling back to the general reasoning model reduces risk. 
This establishes that adaptive routing can outperform either constituent model at inference time, and reframes multi-model collaboration as risk-adaptive policy composition rather than acceleration.

We evaluate Divergence Decoding across diverse architectures, including Qwen and Llama-based pairs, on three authoritative benchmarks: ChemBench~\cite{mirza2024large}, ChemCoTBench~\cite{li2025beyond}, and GPQA~\cite{rein2023gpqa}. Experimental results demonstrate that our framework outperforms both standalone specialists and generalists in most tasks. Ablation studies confirm that JS divergence serves as a superior gating criterion compared to universal cross-vocabulary or log-prob divergence, providing the most robust signal for beneficial fallbacks. Mechanistic analysis further reveals that resampling primarily occurs in the initial 0\%-25\% of the token trajectory, where a distinct entropy shift identifies the boundary between confident expertise and logical uncertainty; notably, the resampled tokens are predominantly natural language connectors rather than domain entities, suggesting the Generalist's role in stabilizing the reasoning framework. Finally, Divergence Decoding maintains high inference efficiency with minimal overhead, offering significant practical utility for real-world scientific discovery.

To summarize, our work makes three contributions. 
\textbf{First}, we introduce a training-free token-level fusion mechanism that combines a science-specialized LLM and a general reasoning LLM during decoding using JS-based gating, with no router training, joint finetuning, or extra supervision. 
\textbf{Second}, we provide a theoretical analysis of this behavior as state-dependent cross-model routing rather than speculative sampling: the analysis shows how JS-based gating can make the combined policy outperform either constituent model when low-divergence states correspond to domain gains and high-divergence states correspond to safer general-model fallback. 
\textbf{Third}, we validate this approach across multiple scientific benchmarks and LLM-series, confirming the $A+B>A\ or\ B$ effect: the collaboration of two complementary large models can be stronger than either model used alone.
\section{Related Work}
\subsection{General reasoning LLMs and domain-specialized scientific LLMs}

Recent progress in large language models has advanced along two largely complementary directions~\cite{li2025decoupled,lv2025prollama,li2023weakly}. On one hand, general-purpose models have improved reasoning ability through large-scale post-training and reinforcement learning: for example, DeepSeek-R1~\cite{guo2025deepseek} demonstrates that RL can elicit stronger long-chain reasoning behaviors such as reflection, verification, and adaptive problem solving, and its distilled variants further transfer such capabilities to smaller dense models. 
On the other hand, domain-specialized scientific LLMs have focused on improving expert knowledge and reasoning in chemistry and biomedicine. ChemDFM~\cite{zhao2024chemdfm} shows that chemistry-focused pretraining and instruction tuning improve chemical understanding and dialogue, while ChemDFM-R~\cite{zhao2025chemdfmr} strengthens chemical reasoning through atomized chemical knowledge and reaction-centered supervision~. Chem-R~\citep{zhao2025developingchemr} extends this line with a multi-stage training recipe combining chemical foundation training, reasoning protocol distillation, and multi-task GRPO to induce more deliberative chemical reasoning. Beyond chemistry, TxGemma~\cite{wang2025txgemma} similarly highlights the promise of domain-oriented LLMs as interactive and explainable models for biomedical reasoning and prediction. 
These two lines of work, however, are typically studied in isolation: general models prioritize broad reasoning robustness, whereas domain models emphasize scientific competence. Our work is motivated by the hypothesis that these strengths are complementary and can be fused at inference time without additional training.

\subsection{Speculative decoding and token-level routing}
Our method is related to both speculative decoding and token-level expert routing. Classical speculative decoding and speculative sampling follow a draft-and-verify paradigm, where a smaller or faster model proposes tokens and a larger target model verifies them to exactly preserve the target distribution; variants such as Medusa~\cite{cai2024medusa} and EAGLE~\cite{li2024eagle,li2024eagle2} improve the drafting mechanism but retain the same acceleration-oriented objective. A separate line of work, including ETR~\cite{zhou2022mixture,chai2024expert}, CITER~\cite{zheng2025citer}, and FusionRoute~\cite{xiong2026token}, studies token-level routing across experts, typically using learned routers for efficiency or trainable coordination. In contrast, we do not aim to reproduce a fixed target model or rely on learned routing; instead, we combine two full-capability LLMs using a training-free rule based on the Jensen--Shannon divergence between their next-token distributions, so that observable disagreement determines whether to follow the specialist or the general model. This design is simple and architecture-agnostic, and is further motivated by evidence that sparse, high-impact token changes can drive reasoning gains~\cite{meng2026sparse}.
\section{Method}

\subsection{Problem Setting}
We consider two autoregressive language models: a {domain-specialized model} \(A\) and a {general reasoning model} \(B\). Model \(A\) is adapted to a target scientific domain through domain-specific pretraining, supervised fine-tuning, and/or preference optimization. Model \(B\) is a strong general-purpose reasoning model with broader inference robustness. Our goal is to combine these two models \emph{at inference time}, without any additional router training or joint finetuning.

We propose a token-level routing policy based on the Jensen--Shannon (JS) divergence between the two models' next-token predictive distributions. The key intuition is that low distributional divergence indicates local agreement, in which case the decoder trusts the domain-specialized model. Conversely, high divergence is treated as a warning signal, and the decoder defers to the general reasoning model to continue generation.

This objective differs from classical speculative decoding. Rather than producing a lossless approximation to a target model, our method defines a new adaptive decoding policy that switches between two models according to their local distributional disagreement.

\subsection{JS-Gated Block Drafting with Sequential Verification}

\begin{figure}[t]
    \centering
    \includegraphics[width=1.0\linewidth]{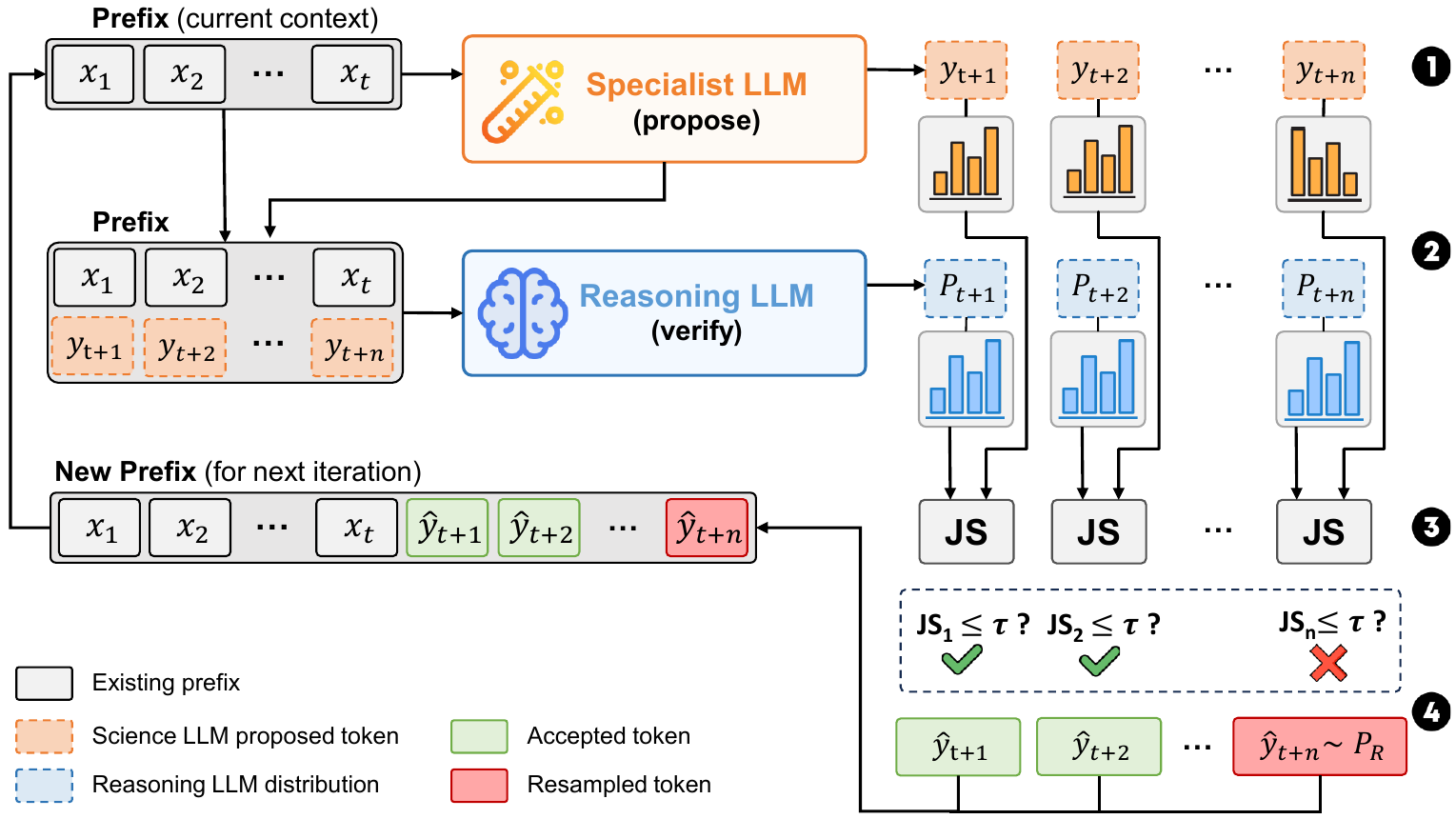}
    \vspace{-0.2in}
    \caption{\textbf{Divergence Decoding framework.} \textbf{(Step 1)} At each iteration, the domain-specialized model drafts a block of candidate tokens and records the corresponding next-token distributions. \textbf{(Step 2)}  The reasoning model then evaluates the same draft prefixes to obtain its predictive distributions.  \textbf{(Step 3)} The decoder computes the JS divergence between the two distributions at each position. \textbf{(Step 4)} If the divergence is below the threshold, the drafted token is accepted; otherwise, the decoder samples a token from the reasoning model and starts a new block.}
    \label{fig:methods}
\end{figure}

Autoregressive decoding traditionally generates tokens one at a time by computing next-token predictions from the current prefix. Our method, however, requires comparing two predictive distributions from two separate models. To avoid querying both models at every token, we implement JS-gated routing with a blockwise draft-then-verify procedure.

Let \(x\) denote the current generation prefix. At each decoding step, the domain-specialized model \(A\) and the reasoning model \(B\) define next-token distributions over the same prefix, and the routing decision is made according to the JS divergence between these two distributions.
Starting from prefix \(x\), model \(A\) drafts a block of \(n\) candidate tokens
$
\hat{\mathbf{y}}_{1:n} = (\hat{y}_1,\ldots,\hat{y}_n).
$
For each position \(1 \leq j \leq n\), define the corresponding draft prefix as
$
x^{(j)} = x \mathbin{\|} \hat{y}_{1:j-1},
x^{(1)} = x.
$
The two models then induce the next-token distributions
$
p_A^{(j)} = p_A(\cdot \mid x^{(j)}),
p_B^{(j)} = p_B(\cdot \mid x^{(j)}).
$
While model \(A\) produces these distributions during drafting, model \(B\) can compute \(p_B^{(1)},\ldots,p_B^{(n)}\) in a single teacher-forced forward pass over the drafted block.

For each position, the decoder computes the Jensen--Shannon divergence $s_j = D_{\mathrm{JS}}\bigl(p_A^{(j)}, p_B^{(j)}\bigr)$, where
\begin{equation*}
D_{\mathrm{JS}}(p,q)
= \frac{1}{2} KL(p \,\|\, m)
+ \frac{1}{2} KL(q \,\|\, m),
\
m = \frac{1}{2}(p+q).
\end{equation*}
We use JS divergence because it is symmetric, bounded, and compares the local predictive distributions rather than only the probability of a realized token, making it suitable for measuring model disagreement during decoding.

A single threshold \(\tau\) determines whether to trust model \(A\) or defer to model \(B\):
\[
g_j = \mathbf{1}[s_j > \tau],
\qquad
\pi_\tau^{(j)}(\cdot \mid x^{(j)})
=
(1-g_j)p_A^{(j)}
+
g_j p_B^{(j)}.
\]
Since \(g_j\) is deterministic, the decoder selects either \(p_A^{(j)}\) or \(p_B^{(j)}\) at each position rather than averaging the two distributions. If \(s_j \leq \tau\), the drafted token \(\hat{y}_j\) is accepted. Otherwise, a replacement token \(y' \sim p_B^{(j)}\) is appended, verification of the current block stops, and the remaining drafted tokens are discarded because they were generated under a prefix that no longer matches the actual output. This asymmetric rule makes model \(A\) the default generator, preserving domain-specific priors, while invoking model \(B\) only when distributional disagreement indicates potential unreliability. The full procedure is shown in Figure~\ref{fig:methods} and Algorithm~\ref{alg:block_drafting}.
In all experiments, both models use the same sampling configuration: temperature \(=1.0\), top-\(p=0.7\), and top-\(k=10\), ensuring that routing decisions are not confounded by model-specific decoding settings.
\begin{algorithm}[t]
\caption{Block Drafting with Sequential Verification}
\small
\label{alg:block_drafting}
\begin{algorithmic}[1]
\State \textbf{Input:} Prompt \(x_{1:t}\), block size \(n\), divergence threshold \(\tau\), maximum length \(T\)
\State \textbf{Output:} Generated sequence \(x\)

\State Initialize the current sequence: \(x \gets x_{1:t}\)

\While{\(|x| < T\) \textbf{and} \(\operatorname{last}(x) \neq \mathrm{EOS}\)}
    \State Store the current prefix: \(x^{(0)} \gets x\)

    \State Draft \(n\) candidate tokens with model \(A\):
    $
    (\hat{\mathbf{y}}_{1:n}, \mathcal{P}_A) \gets \mathrm{Draft}_A(x^{(0)}, n)
    $

    \State Verify the drafted block with model \(B\):
    $
    \mathcal{P}_B \gets \mathrm{Verify}_B(x^{(0)}, \hat{\mathbf{y}}_{1:n})
    $

    \For{\(j = 1\) \textbf{to} \(n\)}
        \State Compute distributional divergence:
        $
        s_j \gets D_{\mathrm{JS}}\bigl(p_A^{(j)}, p_B^{(j)}\bigr)
        $

        \If{\(s_j \leq \tau\)}
            \State Accept the drafted token: \(x \gets x \mathbin{\|} \hat{y}_j\)
        \Else
            \State Sample a replacement token from model \(B\): \(y' \sim p_B^{(j)}\)
            \State Append the replacement token: \(x \gets x \mathbin{\|} y'\)
            \State Stop verification of the current block.
            \State \textbf{break}
        \EndIf

        \If{\(|x| \geq T\) \textbf{or} \(\operatorname{last}(x) = \mathrm{EOS}\)}
            \State Stop generation.
            \State \textbf{break}
        \EndIf
    \EndFor
\EndWhile

\State \Return \(x\)
\end{algorithmic}
\end{algorithm}

\subsection{Theoretical Justification}
We provide a concise justification for why JS-gated routing can outperform using either model alone. Let \(r_t\) denote the oracle target next-token distribution at decoding state \(h_t\), i.e., the ideal conditional distribution that the decoder aims to approximate. We define the JS risk of model \(i\in\{A,B\}\) as
\[
\ell_i(t)=D_{\mathrm{JS}}(r_t,p_i(\cdot\mid h_t)),
\qquad
L(i)=\mathbb{E}[\ell_i(t)].
\]
For the routed policy \(\pi_\tau\), the risk is
\[
\ell_\tau(t)=(1-g_t)\ell_A(t)+g_t\ell_B(t),
\qquad
L(\tau)=\mathbb{E}[\ell_\tau(t)].
\]

The router uses the model disagreement
$
s_t=D_{\mathrm{JS}}\!\left(p_A(\cdot\mid h_t),p_B(\cdot\mid h_t)\right)$
as an observable reliability signal. We assume that the general-purpose model \(B\) is uniformly robust, with
$
\sqrt{D_{\mathrm{JS}}(p_B,r_t)}\le \epsilon_B,
$
while the specialized model \(A\) has a strict advantage in an in-domain subspace but suffers at least \(\epsilon_A\) error outside it. Under these assumptions, the metric property of \(\sqrt{D_{\mathrm{JS}}}\) gives two useful implications.

Intuitively, the JS score separates the states where each model is reliable. When \(s_t>\tau\), the condition \(\tau \ge 4\epsilon_B^2\) together with the \(\epsilon_B\)-robustness of \(B\) implies that \(A\) is farther from the target representation \(r_t\) than \(B\), so routing to \(B\) reduces risk in the high-disagreement region. Conversely, when \(s_t\le\tau\), the condition \(\epsilon_A > \sqrt{\tau}+\epsilon_B\) rules out the possibility that \(A\) is outside its in-domain subspace, since any such state would necessarily yield \(s_t>\tau\); thus low-disagreement states correspond to cases where the specialized model \(A\) is reliable and preferable. Therefore, if \(\tau \ge 4\epsilon_B^2\), \(\epsilon_A > \sqrt{\tau}+\epsilon_B\), and both routing regions occur with non-zero probability, the JS-gated policy strictly outperforms either model used alone:
\[
L(\tau)<\min\{L(A),L(B)\}.
\]
Thus, JS-gated routing uses the specialist in reliable low-disagreement states and falls back to the reasoning model when the specialist is likely to be unreliable. We provide full proof in Appendix~\ref{app:theory_proof}.
\section{Main Experiments on Scientific Benchmarks}
This section provides a systematic evaluation of Divergence Decoding across diverse scientific domains and model architectures. We first establish the effectiveness of our method on scientific reasoning benchmarks, including ChemCoTBench~\cite{li2025beyond}, ChemBench~\cite{mirza2024large}, and GPQA~\cite{rein2023gpqa}. Furthermore, we conduct ablation studies to justify the selection of JS divergence as our core metric. We also analyse the unique token-level patterns emergent during the resampling process and provide a quantitative analysis of inference latency to demonstrate the method’s computational efficiency.

\subsection{Performance in Chemistry Domain}
To evaluate the efficacy and generalizability of our proposed Divergence Decoding strategy, we conduct comprehensive experiments on two domain-specific benchmarks: ChemCoTBench and ChemBench. While ChemCoTBench focuses on specialized molecular editing and understanding through expert-annotated reasoning chains, ChemBench encompasses a broader chemical spectrum, ranging from organic synthesis to materials science. To verify whether Divergence Decoding consistently synergizes domain expertise with logical reasoning across diverse architectures, we investigate two distinct configurations: the Qwen backbone~(ChemDFM-R paired with R1-Distill-Qwen-32B), and the Llama backbone~(ChemR combined with R1-Distill-LLaMA-70B). This dual-backbone setup allows us to assess the strategy's cross-modal alignment performance across varying model scales and pretraining paradigms.

\begin{table}[!ht]
  \caption{Performance on molecule understanding and editing tasks in ChemCoTBench. For accuracy and Tanimoto similarity, higher is better ($\uparrow$); for MAE, lower is better ($\downarrow$).}
  \label{tab:chem_understanding_edit}
  \centering
  \setlength{\tabcolsep}{2pt}
  \small
  \renewcommand{\arraystretch}{1.02}
  \setlength{\tabcolsep}{2.3mm}

  \begin{tabular}{l|cc|cc|c|ccc}
    \toprule
    \multirow{2}{*}{Models}
      & \multicolumn{2}{c|}{Func-Group}
      & \multicolumn{2}{c|}{Scaffold}
      & \multicolumn{1}{c|}{SMILES}
      & \multicolumn{3}{c}{Molecule-Edit} \\
    \cmidrule(r){2-3} \cmidrule(r){4-5} \cmidrule(r){6-6} \cmidrule(r){7-9}
      & FG$\downarrow$
      & Ring$\downarrow$
      & Murcko$\uparrow$
      & Ring-sys$\uparrow$
      & Eq.$\uparrow$
      & Add
      & Delete
      & Sub \\
    \midrule
    \noalign{\vspace{-0.2em}}  
    \multicolumn{9}{c}{\textbf{Backbone: Qwen}} \\
    \noalign{\vspace{-0.2em}}  
    \midrule
    R1-distill-Qwen-32B
      & 0.21 & 1.05 & 0.16 & 0.67 & 0.63
      & 45 & 70 & 28 \\

    ChemDFM-R
      & 0.24 & 0.80 & 0.88 & 0.58 & \textbf{0.86}
      & 40 & 60 & 56 \\

    \makecell[l]{Divergence Decoding}
      & \textbf{0.17} & \textbf{0.55} & \textbf{0.90} & \textbf{0.70} & 0.85
      & \textbf{70} & \textbf{80} & \textbf{66} \\
      
    \midrule
    \noalign{\vspace{-0.2em}}  
    \multicolumn{9}{c}{\textbf{Backbone: Llama}} \\
    \noalign{\vspace{-0.2em}}  
    \midrule
    R1-distill-Llama-70B
      & 0.25 & 0.85 & 0.22 & \textbf{0.65} & \textbf{0.61}
      & 40 & {80} & 51 \\

    Chem-R
      & 0.14 & \textbf{0.30} & 0.58 & 0.62 & 0.54
      & {84} & 70 & \textbf{63} \\

    \makecell[l]{Divergence Decoding}
      & \textbf{0.10} & 0.45 & \textbf{0.61} & \textbf{0.65} & 0.55
      & \textbf{85} & \textbf{80} & 58 \\

    \bottomrule
  \end{tabular}
    \vspace{-6pt}
\end{table}

\begin{table}[!ht]
  \caption{Performance on different chemistry-related datasets for Qwen and Llama backbone methods.}
  \label{tab:chem-benchmark}
  \centering
  \setlength{\tabcolsep}{2pt}
  \small
  \resizebox{\textwidth}{!}{
  \begin{tabular}{lccccccccc}
    \toprule
    Method & \makecell{analytical\\chemistry} & \makecell{chemical\\preference} & \makecell{general\\chemistry} & \makecell{inorganic\\chemistry} & \makecell{materials\\science} & \makecell{organic\\chemistry} & \makecell{physical\\chemistry} & \makecell{technical\\chemistry} & \makecell{toxicity\\and safety} \\
    \midrule

    \multicolumn{10}{l}{\textbf{Backbone: Qwen}} \\
    \cmidrule(lr){1-10}
    R1-distill-Qwen-32B & 0.45 & 0.56 & 0.68 & 0.73 & 0.57 & 0.65 & 0.73 & \textbf{0.70} & 0.37 \\
    ChemDFM-R     & 0.40 & 0.56 & 0.69 & 0.61 & \textbf{0.60} & 0.69 & 0.55 & 0.65 & 0.37 \\
    Divergence Decoding      & \textbf{0.46} & \textbf{0.59} & \textbf{0.74} & \textbf{0.77} & 0.57 & \textbf{0.71} & \textbf{0.75} & 0.69 & \textbf{0.38} \\
    
    \addlinespace[2pt]
    \cmidrule(lr){1-10}
    \multicolumn{10}{l}{\textbf{Backbone: Llama}} \\
    \cmidrule(lr){1-10}
    R1-distill-Llama-70B      & 0.43 & 0.53 & 0.75 & 0.74 & 0.65 & 0.71 & \textbf{0.76} & 0.64 & 0.40 \\
    Chem-R          & 0.28 & 0.50 & 0.22 & 0.35 & 0.36 & 0.43 & 0.21 & 0.34 & 0.18 \\
    Divergence Decoding      & \textbf{0.48} & \textbf{0.55} & \textbf{0.77} & \textbf{0.77} & \textbf{0.69} & \textbf{0.75} & 0.74 & \textbf{0.68} & \textbf{0.41} \\

    \bottomrule
  \end{tabular}
  }
  \vspace{-6pt}
\end{table}

Experimental results in Table.~\ref{tab:chem_understanding_edit} and Table.~\ref{tab:chem-benchmark} indicate that Divergence Decoding outperforms individual constituent models across most tasks, establishing a new performance upper bound. A pivotal observation is the emergent synergy: by bridging the domain-specialized model with a reasoning-intensive counterpart, the collaborative framework not only mitigates the specialized model's logical deficiencies but also surpasses the performance of the standalone reasoning model. This demonstrates that our method effectively leverages the divergence between models to achieve a more precise and grounded chemical reasoning.

\subsection{Performance in Multiple Science Domains}

To comprehensively evaluate our divergence decoding method across various scientific domains, we apply the GPQA-diamond, a challenging benchmark consisting of expert-level questions across domains such as physics, chemistry, and biology. Solving these questions requires the combination of scientific knowledge, problem understanding, and multi-step reasoning. For each discipline, we select a corresponding domain-specialized expert model~(ChemDFM-R, TxGemma, and Raman-1.7B~\cite{raman01}) to provide scientific knowledge, while using the same R1-Distill-Qwen-32B model to guide the reasoning process. Three combinations are tested on domain tasks, respectively. The experimental results show that inference with our divergence decoding strategy, which combines a domain model with a reasoning model, consistently outperforms using either model alone. This trend holds across different scientific domains, suggesting that our method can effectively leverage complementary strengths from domain expertise and reasoning-oriented models.

\begin{table}[t]
  \centering
  \vspace{-0.2in}
  \caption{GPQA accuracy (\%) on chemistry, biology, and physics.}
  \label{tab:gpqa}
  \small
  \setlength{\tabcolsep}{5pt}
  \renewcommand{\arraystretch}{1.1}
  \begin{tabular}{l l c l c c}
    \toprule
    Domain & Science model & Acc. & Target model & Acc. & Divergence Decoding\\
    \midrule
    Chemistry & ChemDFM-R     & 15.28 & R1-distill-Qwen-32B & 23.61 & \textbf{37.50} \\
    Biology   & TxGemma       & 34.72 & R1-distill-Qwen-32B & 52.63 & \textbf{63.16} \\
    Physics   & Raman-1.7B    & 26.74 & R1-distill-Qwen-32B & 61.62 & \textbf{70.93} \\
    \bottomrule
  \end{tabular}
  \vspace{-6pt}
\end{table}

\begin{figure}[!ht]
    \centering
    \includegraphics[width=1.0\linewidth]{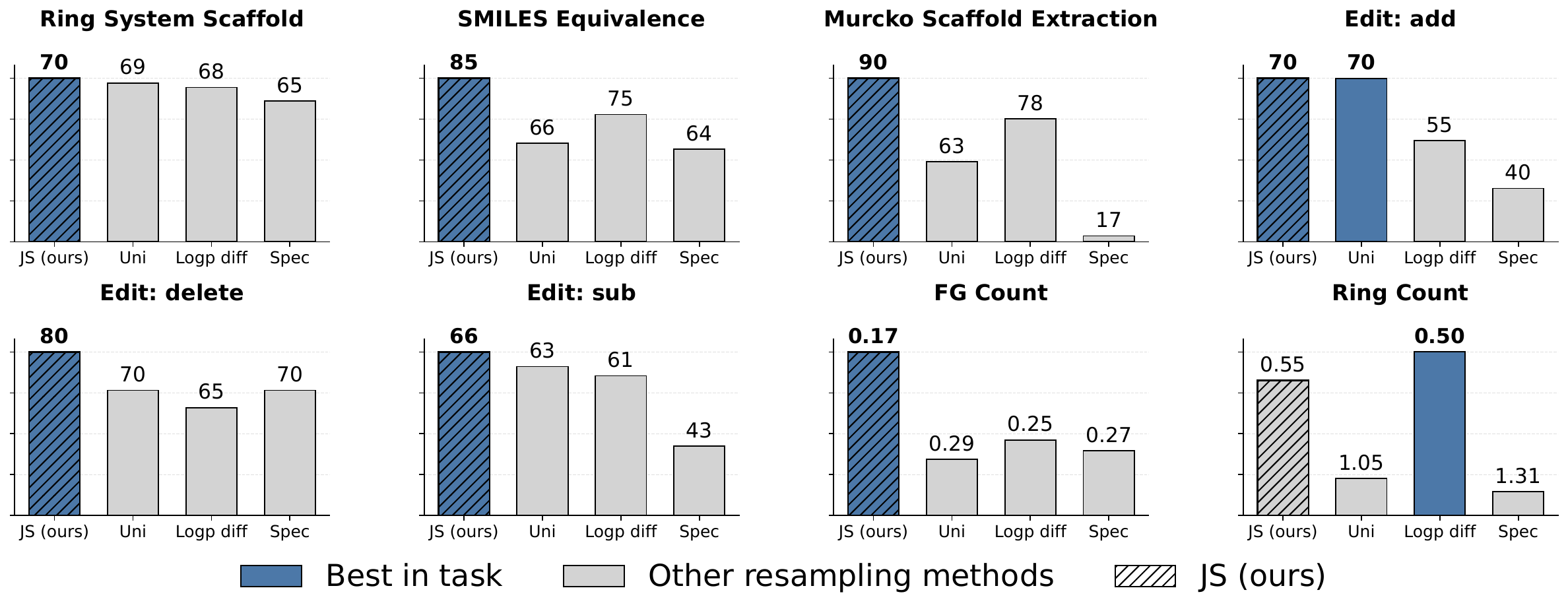}
    \captionsetup{skip=2pt}
    \caption{Comparison of different divergence methods during Divergence Decoding.}
    \label{fig:ablation-signal}
\end{figure}

\subsection{Ablation Study}
The ablation experiments are based on the ChemDFM-R~(specialist) and R1-distill-Qwen-32B~(generalist) on ChemCoTBench. We focus on (1) Divergence Type Comparison; (2) Token analysis during the resampling process in divergence decoding; (3) Hyperparameter analysis.

\textbf{Divergence Type Comparison:} We first compare different divergence types in the resampling stage, including universal cross-vocabulary divergence~\cite{patino2025}, top-10 vocabulary JS divergence, and single-token log-probability difference. These three signals all measure the discrepancy between the predictive distributions of two models at a given decoding position, but differ in the amount of distributional information they exploit. Specifically, \textit{universal cross-vocabulary JS divergence} uses the full vocabulary distribution, retaining the richest discrepancy information while introducing noise and computational cost; \textit{top-10 vocabulary JS divergence} approximates this discrepancy using only the most likely candidate tokens; and \textit{single-token log-probability difference} provides the most limited signal. In Figure.~\ref{fig:ablation-signal}, all three divergence types consistently outperform standard speculative sampling, whose objective is to obtain an unbiased estimate of the target model distribution. This suggests that explicitly leveraging inter-model disagreement can improve decoding quality. 

In Figure.~\ref{fig:ablation-signal}, our divergence decoding using top-10 JS-divergence achieves the best performance on most tasks among all these divergence types.
Overall, these results lead to two observations. First, incorporating divergence types generally improves model performance. Second, top-10 JS divergence provides a favorable trade-off between effectiveness and efficiency, as it is competitive with, and often superior to, full-vocabulary divergence while requiring substantially less computation.

\begin{figure}[!ht]
    \centering
    \includegraphics[width=1.\linewidth]{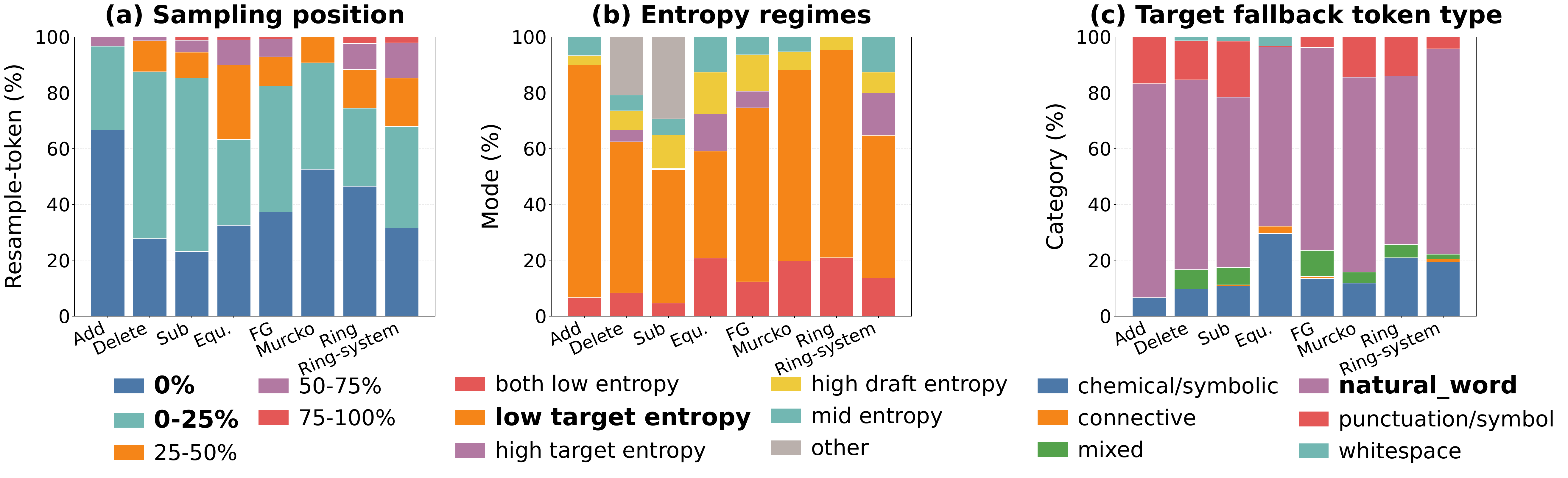}
    \captionsetup{skip=2pt}
    \caption{\textbf{\textit{Resampling Analysis:}} Resampling patterns by position, uncertainty, and token type.}
    \label{fig:resample-study}
\end{figure}

\textbf{Token Analysis during Resampling:} In Figure.~\ref{fig:resample-study}, we further analyze the resampling tokens to investigate how divergence-based decoding affects the reasoning process. First, effective resampling events are highly concentrated in the early stage of the reasoning trajectory, with approximately 77\% occurring at the initial position~(0\%-25\%) of the generated reasoning sequence. This suggests that resampling primarily intervenes during the formation of the reasoning direction, where early errors in identifying functional groups, ring scaffolds, or structural equivalence may otherwise propagate throughout the entire trajectory. Second, the dominant uncertainty regime corresponds to high JS divergence and low entropy of the reasoning LLM, indicating that resampling is mostly triggered by confident disagreement from the reasoning model rather than by reasoning-side uncertainty. The correction signal often comes from a sharp predictive distribution that decisively redirects the domain model. Token-type analysis further shows that the injected fallback tokens are mainly general reasoning words, with a smaller but non-negligible fraction of chemical or symbolic tokens. This implies that resampling adjusts both the global reasoning path and key domain-specific anchors.

\begin{figure}[!ht]
    \centering
    \includegraphics[width=\linewidth]{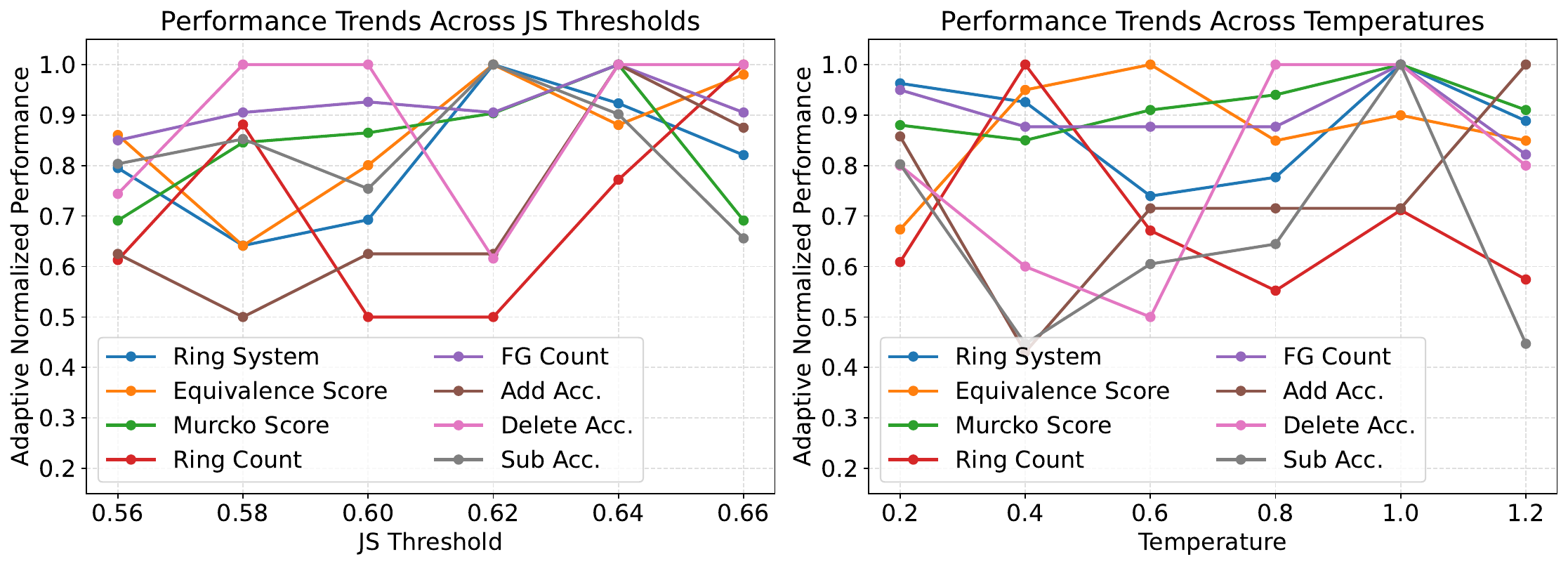}
    \captionsetup{skip=2pt}
    \caption{\textbf{\textit{Sensitivity Analysis:}}For visualization only, metrics with different scales and optimization directions are transformed into task-wise normalized performance scores(higher-is-better). The y-axis reflects normalized relative trends.}
    \label{fig:sensitivity}
\end{figure}

\textbf{Hyperparameter Analysis:} We also conduct a sensitivity analysis on key hyperparameters. The results show that both the JS-divergence threshold and the temperature affect downstream performance. The JS-divergence threshold directly controls the frequency with which reasoning-model tokens are injected, while the temperature changes the sharpness of the predictive distributions, thereby influencing the computed JS divergence and the resulting resampling decisions. Overall, injecting reasoning tokens under a broad range of JS-divergence thresholds brings consistent gains by facilitating the integration of domain knowledge and reasoning ability. Detailed number is presented in Appendix \ref{tab:chem_understanding_edit_temperature}.

\textbf{Efficiency Analysis.}
We also compare the end-to-end generation time on ChemCoTBench. As shown in Table~\ref{tab:chem_time_compressed}, Divergence Decoding achieves faster inference on Mol-Edit tasks but is slower on Mol-Und tasks, resulting in a slight overall latency increase. This task-dependent behavior may stem from different levels of agreement between the specialist and reasoning models: when their predictions align, blockwise drafting and early token acceptance reduce decoding cost; otherwise, frequent fallback to verification interrupts drafting and increases latency. Overall, Divergence Decoding maintains comparable efficiency while enabling stronger capability fusion.

\begin{table}[!ht]
  \centering
  \vspace{-0.2in}
  \caption{Average end-to-end generation time comparison on ChemCoTBench.}
  \label{tab:chem_time_compressed}
  \setlength{\tabcolsep}{5pt}
  \small
  \begin{tabular}{l|ccc}
    \toprule
    Methods
      & Mol-Und Avg.$\downarrow$
      & Mol-Edit Avg.$\downarrow$
      & Overall Avg.$\downarrow$ \\
    \midrule
    Single LLM
      & \textbf{189.02}
      & 218.74
      & \textbf{200.16} \\
    Diver. Decode
      & 233.04
      & \textbf{169.43}
      & 209.19 \\
    \midrule
    Single / DD
      & 0.81$\times$
      & \textbf{1.29$\times$}
      & 0.96$\times$ \\
    \bottomrule
  \end{tabular}
  \vspace{-6pt}
\end{table}

\textbf{Case Study from ChemCoTBench.}
We provide a case study to illustrate how reasoning-model token injection influences the model's early focus and guides the decoding process toward a chemically correct final answer. With the reasoning guidance of R1-Distill Qwen, the specialist model~(ChemDFM) is able to generate high-quality reasoning trajectories with grounded scientific knowledge. More case studies are provided in the Appendix~\ref{app-casestudy}.

\begin{figure}[t]
    \centering
    \includegraphics[width=0.95\linewidth]{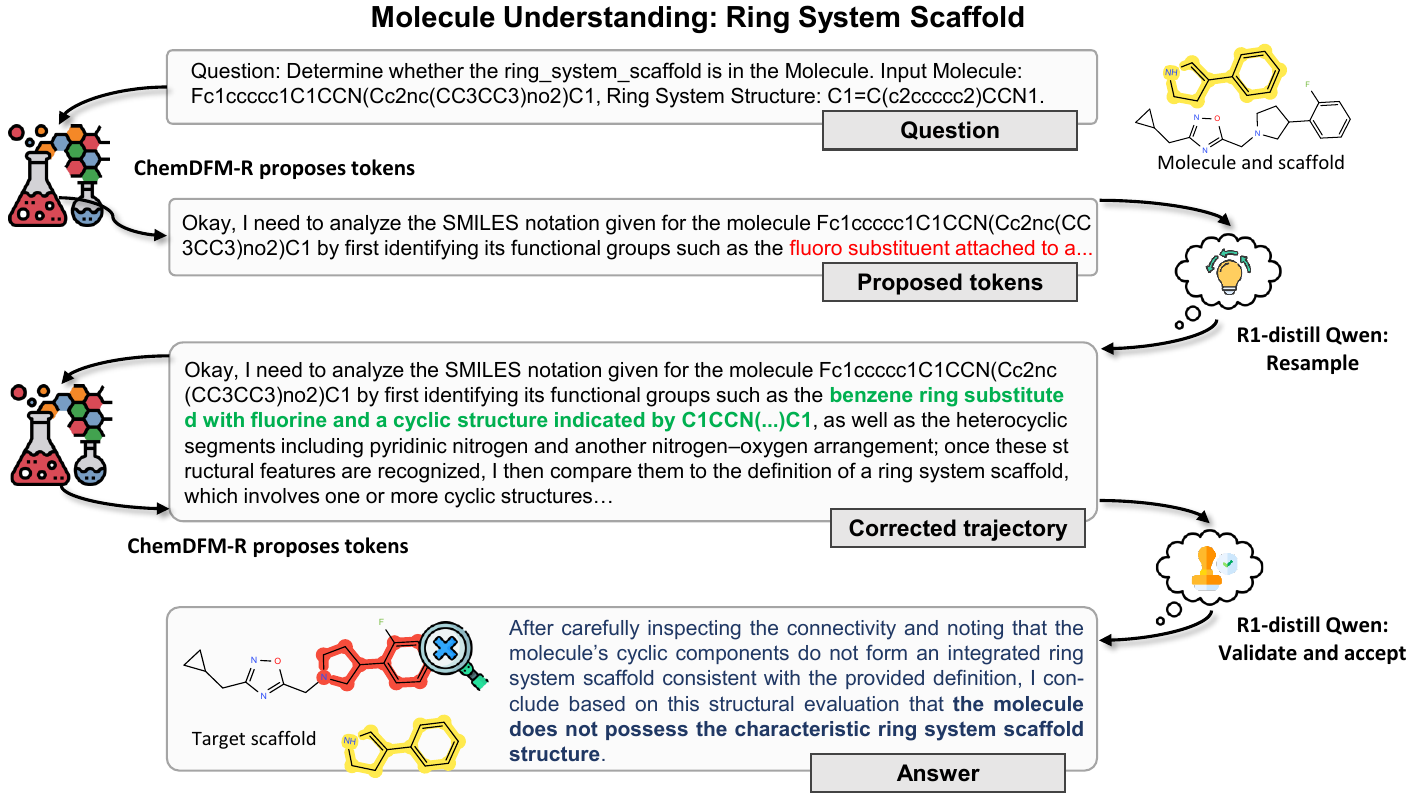}
    \captionsetup{skip=2pt}
    \caption{\textbf{\textit{Case study:}} the divergence decoding between ChemDFM and R1-Distill Qwen affects the reasoning process in the molecule understanding task.}
    \label{fig:casestudy}
\end{figure}

\section{Conclusion}
This paper presents \textit{Divergence Decoding}, a training-free inference-time framework designed to fuse the expertise of domain-specialized LLMs with the logical rigor of general reasoning models. By formulating decoding as a state-dependent routing process, our method adaptively transfers control based on $JS$-divergence, enabling precise token-level integration without the need for additional fine-tuning or supervision. Experimental results across chemistry-centric benchmarks and broader scientific tasks demonstrate that Divergence Decoding surpasses its constituent models in most situations, yielding a synergistic effect where domain knowledge and reasoning capabilities are effectively composed at inference time. Our analyses further reveal that selective, early-stage interventions are pivotal for steering generations toward reliable scientific trajectories. Ultimately, divergence-based adaptive decoding offers a robust and scalable paradigm for constructing high-performance scientific AI systems from existing heterogeneous LLMs.

\clearpage
\bibliographystyle{unsrtnat}
\bibliography{ref}

@article{guo2025deepseek, 
  title={Deepseek-r1: Incentivizing reasoning capability in llms via reinforcement learning},
  author={Guo, Daya and Yang, Dejian and Zhang, Haowei and Song, Junxiao and Wang, Peiyi and Zhu, Qihao and Xu, Runxin and Zhang, Ruoyu and Ma, Shirong and Bi, Xiao and others},
  journal={arXiv preprint arXiv:2501.12948},
  year={2025}
}

@article{zhao2025developingchemr, 
  title={Developing ChemDFM as a large language foundation model for chemistry},
  author={Zhao, Zihan and Ma, Da and Chen, Lu and Sun, Liangtai and Li, Zihao and Xia, Yi and Chen, Bo and Xu, Hongshen and Zhu, Zichen and Zhu, Su and others},
  journal={Cell Reports Physical Science},
  volume={6},
  number={4},
  year={2025},
  publisher={Elsevier}
}

@inproceedings{zhao2024chemdfm, 
  title={Chemdfm: A large language foundation model for chemistry},
  author={Zhao, Zihan and Ma, Da and Chen, Lu and Sun, Liangtai and Li, Zihao and Xia, Yi and Xu, Hongshen and Zhu, Zichen and Zhu, Su and Fan, Shuai and others},
  booktitle={Neurips 2024 Workshop Foundation Models for Science: Progress, Opportunities, and Challenges},
  year={2024}
}

@article{zhao2025chemdfmr,
  title={ChemDFM-R: A Chemical Reasoning LLM Enhanced with Atomized Chemical Knowledge},
  author={Zhao, Zihan and Chen, Bo and Wan, Ziping and Chen, Lu and Lin, Xuanze and Yu, Shiyang and Zhang, Situo and Ma, Da and Zhu, Zichen and Zhang, Danyang and others},
  journal={arXiv preprint arXiv:2507.21990},
  year={2025}
}

@article{wang2025txgemma,
  title={Txgemma: Efficient and agentic llms for therapeutics},
  author={Wang, Eric and Schmidgall, Samuel and Jaeger, Paul F and Zhang, Fan and Pilgrim, Rory and Matias, Yossi and Barral, Joelle and Fleet, David and Azizi, Shekoofeh},
  journal={arXiv preprint arXiv:2504.06196},
  year={2025}
}

@inproceedings{leviathan2023fastspeculative,
  title={Fast inference from transformers via speculative decoding},
  author={Leviathan, Yaniv and Kalman, Matan and Matias, Yossi},
  booktitle={International Conference on Machine Learning},
  pages={19274--19286},
  year={2023},
  organization={PMLR}
}

@article{cai2024medusa,
  title={Medusa: Simple llm inference acceleration framework with multiple decoding heads},
  author={Cai, Tianle and Li, Yuhong and Geng, Zhengyang and Peng, Hongwu and Lee, Jason D and Chen, Deming and Dao, Tri},
  journal={arXiv preprint arXiv:2401.10774},
  year={2024}
}

@article{li2024eagle,
  title={Eagle: Speculative sampling requires rethinking feature uncertainty},
  author={Li, Yuhui and Wei, Fangyun and Zhang, Chao and Zhang, Hongyang},
  journal={arXiv preprint arXiv:2401.15077},
  year={2024}
}

@article{li2025beyond,
  title={Beyond Chemical QA: Evaluating LLM's Chemical Reasoning with Modular Chemical Operations},
  author={Li, Hao and Cao, He and Feng, Bin and Shao, Yanjun and Tang, Xiangru and Yan, Zhiyuan and Yuan, Li and Tian, Yonghong and Li, Yu},
  journal={arXiv preprint arXiv:2505.21318},
  year={2025}
}

@article{li2024eagle2,
  title={Eagle-2: Faster inference of language models with dynamic draft trees, 2024b},
  author={Li, Yuhui and Wei, Fangyun and Zhang, Chao and Zhang, Hongyang},
  journal={URL https://arxiv. org/abs/2406.16858},
  volume={1},
  number={2},
  year={2024}
}

@article{chen2023accelerating,
  title={Accelerating large language model decoding with speculative sampling},
  author={Chen, Charlie and Borgeaud, Sebastian and Irving, Geoffrey and Lespiau, Jean-Baptiste and Sifre, Laurent and Jumper, John},
  journal={arXiv preprint arXiv:2302.01318},
  year={2023}
}

@article{meng2026sparse,
  title={Sparse but critical: A token-level analysis of distributional shifts in RLVR fine-tuning of LLMs},
  author={Meng, Haoming and Huang, Kexin and Wei, Shaohang and Ma, Chiyu and Yang, Shuo and Wang, Xue and Wang, Guoyin and Ding, Bolin and Zhou, Jingren},
  journal={arXiv preprint arXiv:2603.22446},
  year={2026}
}

@article{zhou2022mixture,
  title={Mixture-of-experts with expert choice routing},
  author={Zhou, Yanqi and Lei, Tao and Liu, Hanxiao and Du, Nan and Huang, Yanping and Zhao, Vincent and Dai, Andrew M and Le, Quoc V and Laudon, James and others},
  journal={Advances in Neural Information Processing Systems},
  volume={35},
  pages={7103--7114},
  year={2022}
}

@inproceedings{chai2024expert,
  title={An expert is worth one token: Synergizing multiple expert llms as generalist via expert token routing},
  author={Chai, Ziwei and Wang, Guoyin and Su, Jing and Zhang, Tianjie and Huang, Xuanwen and Wang, Xuwu and Xu, Jingjing and Yuan, Jianbo and Yang, Hongxia and Wu, Fei and others},
  booktitle={Proceedings of the 62nd Annual Meeting of the Association for Computational Linguistics (Volume 1: Long Papers)},
  pages={11385--11396},
  year={2024}
}

@article{xiong2026token,
  title={Token-Level LLM Collaboration via FusionRoute},
  author={Xiong, Nuoya and Zhou, Yuhang and Zeng, Hanqing and Chen, Zhaorun and Huang, Furong and Bi, Shuchao and Zhang, Lizhu and Zhao, Zhuokai},
  journal={arXiv preprint arXiv:2601.05106},
  year={2026}
}

@article{zheng2025citer,
  title={Citer: Collaborative inference for efficient large language model decoding with token-level routing},
  author={Zheng, Wenhao and Chen, Yixiao and Zhang, Weitong and Kundu, Souvik and Li, Yun and Liu, Zhengzhong and Xing, Eric P and Wang, Hongyi and Yao, Huaxiu},
  journal={arXiv preprint arXiv:2502.01976},
  year={2025}
}

@article{mirza2024large,
  title={Are large language models superhuman chemists?},
  author={Mirza, Adrian and Alampara, Nawaf and Kunchapu, Sreekanth and R{\'\i}os-Garc{\'\i}a, Marti{\~n}o and Emoekabu, Benedict and Krishnan, Aswanth and Gupta, Tanya and Schilling-Wilhelmi, Mara and Okereke, Macjonathan and Aneesh, Anagha and others},
  journal={arXiv preprint arXiv:2404.01475},
  year={2024}
}

@article{rein2023gpqa,
  title={Gpqa: A graduate-level google-proof q\&a benchmark},
  author={Rein, David and Hou, Betty Li and Stickland, Asa Cooper and Petty, Jackson and Pang, Richard Yuanzhe and Dirani, Julien and Michael, Julian and Bowman, Samuel R},
  journal={arXiv preprint arXiv:2311.12022},
  year={2023}
}

@article{li2026agentic,
  title={Agentic reinforcement learning empowers next-generation chemical language models for molecular design and synthesis},
  author={Li, Hao and Cao, He and Peng, Shenyao and Liu, Zijing and Feng, Bin and Wang, Yu and Yan, Zhiyuan and Tian, Yonghong and Li, Yu and Yuan, Li},
  journal={arXiv preprint arXiv:2601.17687},
  year={2026}
}

@inproceedings{labrak2024biomistral,
  title={Biomistral: A collection of open-source pretrained large language models for medical domains},
  author={Labrak, Yanis and Bazoge, Adrien and Morin, Emmanuel and Gourraud, Pierre-Antoine and Rouvier, Mickael and Dufour, Richard},
  booktitle={Findings of the association for computational linguistics: acl 2024},
  pages={5848--5864},
  year={2024}
}

@article{bai2025intern,
  title={Intern-s1: A scientific multimodal foundation model},
  author={Bai, Lei and Cai, Zhongrui and Cao, Yuhang and Cao, Maosong and Cao, Weihan and Chen, Chiyu and Chen, Haojiong and Chen, Kai and Chen, Pengcheng and Chen, Ying and others},
  journal={arXiv preprint arXiv:2508.15763},
  year={2025}
}

@inproceedings{zhang2024comprehensive,
  title={A comprehensive survey of scientific large language models and their applications in scientific discovery},
  author={Zhang, Yu and Chen, Xiusi and Jin, Bowen and Wang, Sheng and Ji, Shuiwang and Wang, Wei and Han, Jiawei},
  booktitle={Proceedings of the 2024 Conference on Empirical Methods in Natural Language Processing},
  pages={8783--8817},
  year={2024}
}

@article{hu2025survey,
  title={A survey of scientific large language models: From data foundations to agent frontiers},
  author={Hu, Ming and Ma, Chenglong and Li, Wei and Xu, Wanghan and Wu, Jiamin and Hu, Jucheng and Li, Tianbin and Zhuang, Guohang and Liu, Jiaqi and Lu, Yingzhou and others},
  journal={arXiv preprint arXiv:2508.21148},
  year={2025}
}

@article{bai2023qwen,
  title={Qwen technical report},
  author={Bai, Jinze and Bai, Shuai and Chu, Yunfei and Cui, Zeyu and Dang, Kai and Deng, Xiaodong and Fan, Yang and Ge, Wenbin and Han, Yu and Huang, Fei and others},
  journal={arXiv preprint arXiv:2309.16609},
  year={2023}
}

@article{qwen2024qwen2,
  title={Qwen2. 5 technical report},
  author={Qwen, A Yang and Yang, Baosong and Zhang, Beichen and Hui, Binyuan and Zheng, Bo and Yu, Bowen and Li, Chengpeng and Liu, Dayiheng and Huang, Fei and Wei, Haoran and others},
  journal={arXiv preprint arXiv:2412.15115},
  year={2024}
}

@article{yang2025qwen3,
  title={Qwen3 technical report},
  author={Yang, An and Li, Anfeng and Yang, Baosong and Zhang, Beichen and Hui, Binyuan and Zheng, Bo and Yu, Bowen and Gao, Chang and Huang, Chengen and Lv, Chenxu and others},
  journal={arXiv preprint arXiv:2505.09388},
  year={2025}
}

@misc{qwen3.5,
    title  = {{Qwen3.5}: Towards Native Multimodal Agents},
    author = {{Qwen Team}},
    year   = {2026},
    month  = {February},
    url    = {https://qwen.ai/blog?id=qwen3.5}
}

@article{touvron2023llama,
  title={Llama: Open and efficient foundation language models},
  author={Touvron, Hugo and Lavril, Thibaut and Izacard, Gautier and Martinet, Xavier and Lachaux, Marie-Anne and Lacroix, Timoth{\'e}e and Rozi{\`e}re, Baptiste and Goyal, Naman and Hambro, Eric and Azhar, Faisal and others},
  journal={arXiv preprint arXiv:2302.13971},
  year={2023}
}

@article{grattafiori2024llama,
  title={The llama 3 herd of models},
  author={Grattafiori, Aaron and Dubey, Abhimanyu and Jauhri, Abhinav and Pandey, Abhinav and Kadian, Abhishek and Al-Dahle, Ahmad and Letman, Aiesha and Mathur, Akhil and Schelten, Alan and Vaughan, Alex and others},
  journal={arXiv preprint arXiv:2407.21783},
  year={2024}
}

@article{jiang2023mistral,
  title={Mistral 7b. arxiv},
  author={Jiang, Albert Q and Sablayrolles, A and Mensch, A and Bamford, C and Chaplot, D Singh and Casas, Ddl and Bressand, F and Lengyel, G and Lample, G and Saulnier, L and others},
  journal={arXiv preprint arXiv:2310.06825},
  volume={10},
  pages={3},
  year={2023}
}

@article{openai2023gpt4,
  title={GPT-4 Technical Report},
  author={OpenAI},
  journal={arXiv preprint arXiv:2303.08774},
  year={2023}
}

@article{Ouyang2022instructgpt,
  title={Training language models to follow instructions with human feedback},
  author={Long Ouyang and Jeff Wu and Xu Jiang and Diogo Almeida and Carroll L. Wainwright and Pamela Mishkin and Chong Zhang and Sandhini Agarwal and Katarina Slama and Alex Ray and John Schulman and Jacob Hilton and Fraser Kelton and Luke Miller and Maddie Simens and Amanda Askell and Peter Welinder and Paul Christiano and Jan Leike and Ryan Lowe},
  journal={arXiv preprint arXiv:2203.02155},
  year={2022}
}

@article{li2025decoupled,
  title={Decoupled peak property learning for efficient and interpretable electronic circular dichroism spectrum prediction},
  author={Li, Hao and Long, Da and Yuan, Li and Wang, Yu and Tian, Yonghong and Wang, Xinchang and Mo, Fanyang},
  journal={Nature Computational Science},
  volume={5},
  number={3},
  pages={234--244},
  year={2025},
  publisher={Nature Publishing Group US New York}
}

@article{lv2025prollama,
  title={Prollama: A protein large language model for multi-task protein language processing},
  author={Lv, Liuzhenghao and Lin, Zongying and Li, Hao and Liu, Yuyang and Cui, Jiaxi and Chen, Calvin Yu-Chian and Yuan, Li and Tian, Yonghong},
  journal={IEEE Transactions on Artificial Intelligence},
  year={2025},
  publisher={IEEE}
}

@article{li2023weakly,
  title={Weakly-supervised 3d spatial reasoning for text-based visual question answering},
  author={Li, Hao and Huang, Jinfa and Jin, Peng and Song, Guoli and Wu, Qi and Chen, Jie},
  journal={IEEE Transactions on Image Processing},
  volume={32},
  pages={3367--3382},
  year={2023},
  publisher={IEEE}
}

@inproceedings{li2024freestyleret,
  title={Freestyleret: retrieving images from style-diversified queries},
  author={Li, Hao and Jia, Yanhao and Jin, Peng and Cheng, Zesen and Li, Kehan and Sui, Jialu and Liu, Chang and Yuan, Li},
  booktitle={European Conference on Computer Vision},
  pages={258--274},
  year={2024},
  organization={Springer}
}

@misc{claude2,
    title={Model Card and Evaluations for Claude Models},
    author={Anthropic},
    year={2023},
    url={https://www.anthropic.com/news/claude-2}
}

@misc{claude3,
    title={Introducing the next generation of Claude},
    author={Anthropic},
    year={2024},
    url={https://www.anthropic.com/news/claude-3-family}
}

@article{geminiteam2023gemini,
      title={Gemini: A Family of Highly Capable Multimodal Models}, 
      author={{Gemini Team Google}},
      year={2023},
      journal={arXiv preprint arXiv:2312.11805},
}

@misc{gemma,
    title={Gemma Open Models},
    author={Google},
    year={2024},
    url={https://ai.google.dev/gemma}
}

@article{phan2025humanity,
  title={Humanity's last exam},
  author={Phan, Long and Gatti, Alice and Han, Ziwen and Li, Nathaniel and Hu, Josephina and Zhang, Hugh and Zhang, Chen Bo Calvin and Shaaban, Mohamed and Ling, John and Shi, Sean and others},
  journal={arXiv preprint arXiv:2501.14249},
  year={2025}
}

@article{hendrycks2020measuring,
  title={Measuring massive multitask language understanding},
  author={Hendrycks, Dan and Burns, Collin and Basart, Steven and Zou, Andy and Mazeika, Mantas and Song, Dawn and Steinhardt, Jacob},
  journal={arXiv preprint arXiv:2009.03300},
  year={2020}
}

@article{wang2024mmlu,
  title={Mmlu-pro: A more robust and challenging multi-task language understanding benchmark},
  author={Wang, Yubo and Ma, Xueguang and Zhang, Ge and Ni, Yuansheng and Chandra, Abhranil and Guo, Shiguang and Ren, Weiming and Arulraj, Aaran and He, Xuan and Jiang, Ziyan and others},
  journal={Advances in Neural Information Processing Systems},
  volume={37},
  pages={95266--95290},
  year={2024}
}

@article{auer2023sciqa,
  title={The sciqa scientific question answering benchmark for scholarly knowledge},
  author={Auer, S{\"o}ren and Barone, Dante AC and Bartz, Cassiano and Cortes, Eduardo G and Jaradeh, Mohamad Yaser and Karras, Oliver and Koubarakis, Manolis and Mouromtsev, Dmitry and Pliukhin, Dmitrii and Radyush, Daniil and others},
  journal={Scientific Reports},
  volume={13},
  number={1},
  pages={7240},
  year={2023},
  publisher={Nature Publishing Group UK London}
}

@inproceedings{lu2024moleculeqa,
  title={Moleculeqa: A dataset to evaluate factual accuracy in molecular comprehension},
  author={Lu, Xingyu and Cao, He and Liu, Zijing and Bai, Shengyuan and Chen, Leqing and Yao, Yuan and Zheng, Hai-Tao and Li, Yu},
  booktitle={Findings of the Association for Computational Linguistics: EMNLP 2024},
  pages={3769--3789},
  year={2024}
}

@misc{raman01,
  author = {Diddigam, Sai Praneeth},
  title = {Raman-01: Compact RL-Enhanced Physics Solver},
  year = {2025},
  howpublished = {\url{https://huggingface.co/think-a-tron/raman-01-1.7B}},
  note = {RL (GRPO) finetuned on Qwen3-1.7B}
}

@misc{patino2025,
  title={Unlocking On-Policy Distillation for Any Model Family},
  author={Carlos Miguel Patiño and Kashif Rasul and Quentin Gallouédec and Ben Burtenshaw and Sergio Paniego and Vaibhav Srivastav and Thibaud Frere and Ed Beeching and Lewis Tunstall and Leandro von Werra and Thomas Wolf},
  year={2025},
}

\clearpage

\appendix
\section{Theoretical Analysis}
In this section, we provide sufficient conditions under which a JS-thresholded routing policy outperforms both constituent models. This could be an explaintion for why js-routed divergence decoding outperforms both models under general conditions.
\subsection{Theoretical Setup}
Let the decoding state at step $t$ be the prefix
\begin{equation*}
    h_t = (x_1,\ldots,x_{t-1}).
\end{equation*}
Consider a domain-specialized model $A$ and a general-purpose model $B$.
We denote by $p_A(\cdot \mid h_t)$ and $p_B(\cdot \mid h_t)$ the next-token distributions of models $A$ and $B$, respectively, and by $r_t(\cdot)$ an unknown target distribution that represents the desired conditional distribution at state $h_t$. Throughout, we suppress the conditioning on $h_t$ when no ambiguity arises.

\paragraph{Model assumptions.} We consider general assumptions on the domain-specialized model $A$ and the general-purpose model $B$. For a powerful general-purpose model, we assume that it is capable enough to fit the domain-specific task with a small bounded error. Furthermore, we assume there is an in-domain state subspace where the domain-specialized model fundamentally outperforms the general-purpose model. 

\begin{assumption}[Robustness of the General-Purpose Model]
\label{assump:general_robustness}
There exists a constant $\epsilon_B > 0$, such that for every decoding state $h_t$, the general-purpose model $B$ maintains a bounded distance to the target distribution $r_t$:
\begin{equation*}
    \sqrt{D_{\mathrm{JS}}(p_B, r_t)} \le \epsilon_B.
\end{equation*}
\end{assumption}

\begin{assumption}[In-Domain Subspace Advantage]
\label{assump:in_domain_advantage}
There exist a constant $\delta>0$ and an in-domain state subspace $\mathcal{H}_{\mathrm{in}}$ for the domain-specialized model $A$ such that for any decoding state $h_t \in \mathcal{H}_{\mathrm{in}}$, model $A$ maintains a strict precision advantage over model $B$, such that the difference in their JS distances to the target distribution satisfies:
\begin{equation*}
    \sqrt{D_{\mathrm{JS}}(p_B, r_t)} - \sqrt{D_{\mathrm{JS}}(p_A, r_t)} \ge \delta.
\end{equation*}
Moreover, there exists a constant $\epsilon_A$ depending on $\delta$, $\forall h_t \notin \mathcal{H}_{\mathrm{in}}$, 
\begin{equation*}
    \sqrt{D_{\mathrm{JS}}(p_A, r_t)} \ge \epsilon_A.
\end{equation*}
\end{assumption}

\paragraph{JS-gated cross-model routing.} Define the observable disagreement signal as the JS divergence between $p_A(\cdot \mid h_t)$ and $p_B(\cdot \mid h_t)$
\begin{equation*}
    s_t = D_\mathrm{JS}\left(p_A(\cdot \mid h_t), p_B(\cdot \mid h_t)\right).
\end{equation*}
Given a threshold $\tau > 0$, we define the hard gate
\begin{equation*}
    g_t = \mathbf{1}[s_t > \tau].
\end{equation*}
Hence, the routed policy is defined as
\begin{align}
    \label{eq: JS-gated}
    \pi_\tau(\cdot \mid h_t)=(1-g_t) p_A(\cdot \mid h_t) + g_t p_B(\cdot \mid h_t). \tag{JS-gated routing}
\end{align}
Thus, the decoder accepts the specialized model in low-disagreement states and falls back to the general model in high-disagreement states.
\paragraph{Goal.}
We measure the risk with respect to the target distribution using 
\begin{equation*}
    \ell_i(t) = D_{\rm JS}(r_t,p_i),\quad i\in \{A,B\}.
\end{equation*}
We further define the expected risk
\begin{equation*}
    L(i)=\mathbb{E}[\ell_i(t)].
\end{equation*}
Hence, the routed loss is
\begin{equation*}
    \ell_\tau(t) = (1-g_t)\ell_A(t) + g_t \ell_B(t).
\end{equation*}
The expected risk for the JS-gated routing is
\begin{equation*}
    L(\tau) = \mathbb{E}[\ell_\tau(t)].
\end{equation*}
The goal is to show that under general conditions, $L(\tau)$ can be smaller than $L(A)$ and $L(B)$. Namely, JS-gated cross-model routing can perform better than specialized model $A$ and general model $B$.

\subsection{Theoretical Proof}
\label{app:theory_proof}
To connect the risk metric with our assumptions, we define the risk difference $\Delta_t = \ell_B(t) - \ell_A(t)$. 
Remark that we only consider the task is non-trivial, i.e., $\mathbb{P}(s_t > \tau) > 0, \ \mathbb{P}(s_t \le \tau) > 0$.
If the task were entirely within the specialized domain ($\mathbb{P}(s_t > \tau) = 0$) or entirely out-of-domain ($\mathbb{P}(s_t \le \tau) = 0$), routing would be unnecessary, as a single model would suffice.

\begin{theorem}
\label{thm:routing_superiority}
    Under Assumption \ref{assump:general_robustness} and Assumption \ref{assump:in_domain_advantage}, if $\tau,\epsilon_B,\epsilon_A$ satisfy
    \begin{equation*}
        \tau \geq 4\epsilon_B^2, \quad \epsilon_A > \sqrt{\tau}+\epsilon_B,
    \end{equation*}
    then the expected risk of the JS-gated cross-model routing policy is strictly less than that of both individual models:
    \begin{equation*}
    L(\tau) < \min(L(A),L(B)).
\end{equation*}
\end{theorem}
\begin{proof}
    By definition of the routed policy, the loss at step $t$ is:
    \begin{equation*}
        \ell_\tau(t) = \ell_A(t) + g_t(\ell_B(t)-\ell_A(t)) = \ell_A(t) + g_t \Delta_t.
    \end{equation*}
    Taking the expectation yields the risk difference relative to model $A$:
    \begin{equation*}
        L(\tau) - L(A) = \mathbb{E}[g_t \Delta_t] = \mathbb{P}(s_t > \tau)\mathbb{E}[\Delta_t \mid s_t > \tau].
    \end{equation*}
    Similarly, expressing the routed loss in terms of model $B$ gives $\ell_\tau(t) = \ell_B(t) - (1-g_t)\Delta_t$, which implies:
    \begin{equation*}
        L(\tau) - L(B) = -\mathbb{E}[(1-g_t)\Delta_t] = -\mathbb{P}(s_t \le \tau)\mathbb{E}[\Delta_t \mid s_t \le \tau].
    \end{equation*}
    To prove $L(\tau) < \min(L(A), L(B))$, it suffices to show that $\mathbb{E}[\Delta_t \mid s_t > \tau] < 0$ and $\mathbb{E}[\Delta_t \mid s_t \le \tau] > 0$. We mainly leverage the fact that the square root of the Jensen-Shannon divergence, $\sqrt{D_{\mathrm{JS}}(P, Q)}$, is a true mathematical metric and therefore satisfies the triangle inequality.
    \begin{itemize}
        \item \textbf{Case 1: High-disagreement region ($s_t > \tau$)}

        Noticing that
        \begin{equation*}
            \sqrt{s_t} = \sqrt{D_\mathrm{JS}\left(p_A(\cdot \mid h_t), p_B(\cdot \mid h_t)\right)} \leq \sqrt{D_{\mathrm{JS}}(p_A, r_t)} + \sqrt{D_{\mathrm{JS}}(p_B, r_t)},
        \end{equation*}
        and by Assumption \ref{assump:general_robustness},
        \begin{equation*}
            \sqrt{D_{\mathrm{JS}}(p_B, r_t)} \le \epsilon_B,
        \end{equation*}
        we have
        \begin{equation*}
            \sqrt{D_{\mathrm{JS}}(p_A, r_t)} \ge \sqrt{s_t} - \sqrt{D_{\mathrm{JS}}(p_B, r_t)} > 2\epsilon_B - \epsilon_B = \epsilon_B,
        \end{equation*}
        where we use $s_t>\tau\geq 4\epsilon_B^2$. This strictly guarantees $\sqrt{D_{\mathrm{JS}}(p_A, r_t)} > \sqrt{D_{\mathrm{JS}}(p_B, r_t)}$. Consequently, model $A$ diverges further from the target distribution than model $B$, yielding a negative risk difference $\Delta_t < 0$. Therefore, $\mathbb{E}[\Delta_t \mid s_t > \tau] < 0$.
        \item \textbf{Case 2: Low-disagreement region ($s_t \le \tau$)}
        
        We analyze the state space by bounding the divergence in the out-of-domain region. For any out-of-domain state $h_t \notin \mathcal{H}_{\mathrm{in}}$, applying the reverse triangle inequality on the JS divergence metric yields:
        \begin{equation*}
            \sqrt{s_t} = \sqrt{D_{\mathrm{JS}}(p_A, p_B)} \ge \left| \sqrt{D_{\mathrm{JS}}(p_A, r_t)} - \sqrt{D_{\mathrm{JS}}(p_B, r_t)} \right|.
        \end{equation*}
        By Assumption \ref{assump:in_domain_advantage}, the specialized model degrades significantly out-of-domain such that $\sqrt{D_{\mathrm{JS}}(p_A, r_t)} \ge \epsilon_A$, and by Assumption \ref{assump:general_robustness}, the general model maintains $\sqrt{D_{\mathrm{JS}}(p_B, r_t)} \le \epsilon_B$. Under the condition that the parameter $\epsilon_A$ satisfies $\epsilon_A > \sqrt{\tau} + \epsilon_B$, we obtain the lower bound for the out-of-domain disagreement signal:
        \begin{equation*}
            \sqrt{s_t} \ge \epsilon_A - \epsilon_B > \sqrt{\tau}.
        \end{equation*}
        This strictly implies $s_t > \tau$ for all $h_t \notin \mathcal{H}_{\mathrm{in}}$. By contraposition, observing a low-disagreement signal mathematically guarantees that the decoding state resides within the in-domain subspace:
        \begin{equation*}
            s_t \le \tau \implies h_t \in \mathcal{H}_{\mathrm{in}}.
        \end{equation*}
        Consequently, within this low-disagreement region, the strict precision advantage of model $A$ from Assumption \ref{assump:in_domain_advantage} definitively applies. For any $h_t \in \mathcal{H}_{\mathrm{in}}$, we have $\sqrt{D_{\mathrm{JS}}(p_B, r_t)} - \sqrt{D_{\mathrm{JS}}(p_A, r_t)} \ge \delta > 0$. The risk difference expands as:
        \begin{equation*}
            \Delta_t = \ell_B(t) - \ell_A(t) = \left(\sqrt{\ell_B(t)} - \sqrt{\ell_A(t)}\right)\left(\sqrt{\ell_B(t)} + \sqrt{\ell_A(t)}\right) \ge \delta \cdot (\delta + 0) = \delta^2 > 0.
        \end{equation*}
        Since $\Delta_t$ is strictly bounded below by $\delta^2 > 0$ for all states satisfying $s_t \le \tau$, its conditional expectation is strictly positive:
        \begin{equation*}
            \mathbb{E}[\Delta_t \mid s_t \le \tau] \ge \delta^2 > 0.
        \end{equation*}
    \end{itemize}
    Together with Case 1 and Case 2, the proof is completed.
\end{proof}

\section{Implementation Details.}
First, we set the block size to 10 in all experiments. That is, the science model samples a block of ten tokens at each drafting step, and the reasoning LLM then computes the predictive distributions at the corresponding ten token positions.

Second, for the JS threshold selection, we consider that science-domain models such as ChemDFM-R~\cite{zhao2025chemdfmr} already possess strong molecular understanding. Therefore, for the evaluation on ChemCOTBench~\cite{li2025beyond}, we follow the setting of~\cite{meng2026sparse} and set the JS threshold to 0.65 with a temperature of 1.0. Using a temperature of 1.0 mitigates the influence of model-specific decoding configurations on routing decisions, while a JS threshold of 0.65 restricts the proportion of injected tokens to approximately 3\%--5\% of the total generated tokens. This allows the decoder to enhance reasoning capability while largely preserving the model's domain-specific knowledge, given that the theoretical upper bound of the JS divergence is around 0.70.

For ChemBench, considering that the domain knowledge acquired by science-domain models during training is primarily concentrated on molecular and reaction understanding~\cite{li2024freestyleret}, whereas the benchmark places greater emphasis on reasoning capability, we lower the JS threshold to around 0.2 in this evaluation. This setting permits resampling at more token positions, thereby facilitating a better integration of domain-specific knowledge and general reasoning ability.

Third, for the computation of top-10 JS divergence, the distributions proposed by the two models, denoted as $p$ and $q$, are almost always not exactly identical. We therefore first construct an intermediate support $m$ that contains all tokens appearing in either $p$ or $q$. For tokens that appear in $m$ but are absent from either $p$ or $q$, we extend the corresponding distribution and assign probability zero to those missing tokens. We then compute the KL divergence from $p$ to $m$ and from $q$ to $m$, respectively, to obtain the JS divergence.

When the prefixes differ, feeding the same text into another model may lead to inconsistent tokenization. To make the distributions from the two models comparable, we first perform sequence-level alignment. We then merge and normalize the probabilities over the aligned segments, and compute the JS divergence only on the aligned parts.

We note that this sequence-alignment procedure is highly similar to part of the method used in universal cross-vocabulary JS divergence~\cite{patino2025}, so we include it as a comparison in the ablation study in the main paper. Universal cross-vocabulary JS divergence is designed for model distillation when tokenizers are inconsistent. It first performs sequence alignment and probability merging. Then, for the two distributions, it directly computes JS divergence over the alignable segments, pads the unaligned tails into equal-length vectors, computes an additional distance over these tails, and adds the two terms as the final distributional distance. In practice, to avoid excessive computation, we use the top-200 distributions to compute this universal JS divergence and compare its performance with our method.

Similarly, the thresholds for the log-probability difference baseline and the universal JS divergence baseline are chosen so that only a very small fraction of tokens are resampled. Their thresholds are set to 5.0 and 0.11, respectively.
\section{Sensitivity Analysis}
\label{app:sen}
We compare performance on molecule understanding and editing tasks in ChemCoTBench under different setting of js-thresholds and different temperatures. Detailed results are presented in table \ref{tab:chem_understanding_edit_threshold} and \ref{tab:chem_understanding_edit_temperature}.
\begin{table*}[!ht]
  \caption{Sensitivity analysis of JS-threshold selection.}
  \label{tab:chem_understanding_edit_threshold}
  \centering
  \setlength{\tabcolsep}{3pt}
  \small
  \renewcommand{\arraystretch}{1.02}
  \setlength{\tabcolsep}{2.3mm}

  \begin{tabular}{l|cc|cc|c|ccc}
    \toprule
    \multirow{2}{*}{JS threshold}
      & \multicolumn{2}{c|}{Func-Group}
      & \multicolumn{2}{c|}{Scaffold}
      & \multicolumn{1}{c|}{SMILES}
      & \multicolumn{3}{c}{Molecule-Edit} \\
    \cmidrule(r){2-3} \cmidrule(r){4-5} \cmidrule(r){6-6} \cmidrule(r){7-9}
      & FG$\downarrow$
      & Ring$\downarrow$
      & Murcko$\uparrow$
      & Ring-sys$\uparrow$
      & Eq.$\uparrow$
      & Add$\uparrow$
      & Delete$\uparrow$
      & Sub$\uparrow$ \\
    \midrule

    \multicolumn{9}{c}{\textbf{Backbone: Qwen}} \\
    \midrule

    \makecell[l]{0.56}
      & 0.24 & 0.89 & 0.74 & 0.62 & 0.78
      & 55 & 70 & 58 \\

    \makecell[l]{0.58}
      & 0.21 & 0.63 & 0.82 & 0.56 & 0.67
      & 50 & \textbf{80} & 60 \\

    \makecell[l]{0.60}
      & 0.20 & 1.05 & 0.83 & 0.58 & 0.75
      & 55 & 70 & 56 \\

    \makecell[l]{0.62}
      & 0.27 & 1.05 & 0.85 & \textbf{0.68} & 0.85
      & 55 & 65 & \textbf{62} \\

    \makecell[l]{0.64}
      & \textbf{0.19} & 0.72 & \textbf{0.88} & 0.67 & 0.79
      & 60 & \textbf{80} & 60 \\

    \makecell[l]{0.66}
      & 0.21 & \textbf{0.55} & 0.74 & 0.63 & 0.84
      & \textbf{65} & 70 & 52 \\

    \bottomrule
  \end{tabular}
  \vspace{-6pt}
\end{table*}

\begin{table*}[!ht]
  \caption{Sensitivity analysis of temperature selection.}
  \label{tab:chem_understanding_edit_temperature}
  \centering
  \setlength{\tabcolsep}{3pt}
  \small
  \renewcommand{\arraystretch}{1.02}
  \setlength{\tabcolsep}{2.3mm}

  \begin{tabular}{l|cc|cc|c|ccc}
    \toprule
    \multirow{2}{*}{Temperature}
      & \multicolumn{2}{c|}{Func-Group}
      & \multicolumn{2}{c|}{Scaffold}
      & \multicolumn{1}{c|}{SMILES}
      & \multicolumn{3}{c}{Molecule-Edit} \\
    \cmidrule(r){2-3} \cmidrule(r){4-5} \cmidrule(r){6-6} \cmidrule(r){7-9}
      & FG$\downarrow$
      & Ring$\downarrow$
      & Murcko$\uparrow$
      & Ring-sys$\uparrow$
      & Eq.$\uparrow$
      & Add$\uparrow$
      & Delete$\uparrow$
      & Sub$\uparrow$ \\
    \midrule

    \multicolumn{9}{c}{\textbf{Backbone: Qwen}} \\
    \midrule

    \makecell[l]{0.2}
      & 0.1978 & 0.80 & 0.84 & 0.66 & 0.70
      & 65 & 70 & 55 \\

    \makecell[l]{0.4}
      & 0.21 & \textbf{0.55} & 0.83 & 0.65 & 0.81
      & 50 & 60 & 46 \\

    \makecell[l]{0.6}
      & 0.21 & 0.75 & 0.85 & 0.60 & \textbf{0.83}
      & 60 & 55 & 50 \\

    \makecell[l]{0.8}
      & 0.21 & 0.85 & 0.86 & 0.61 & 0.77
      & 60 & \textbf{80} & 51 \\

    \makecell[l]{1.0}
      & \textbf{0.19} & 0.72 & \textbf{0.88} & \textbf{0.67} & 0.79
      & 60 & \textbf{80} & \textbf{60} \\

    \makecell[l]{1.2}
      & 0.22 & 0.83 & 0.85 & 0.64 & 0.77
      & \textbf{70} & 70 & 46 \\

    \bottomrule
  \end{tabular}
  \vspace{-6pt}
\end{table*}
\section{Compute Resources and Model/Asset Details.}
All experiments were conducted on a single machine equipped with four NVIDIA A800 GPUs, each with 80GB memory. The detailed running time is listed in Table \ref{tab:chem_time_comparison}.
\begin{table*}[!ht]
  \centering
    \caption{End-to-end generation time in ChemCoTBench.}
  \label{tab:chem_time_comparison}
  \setlength{\tabcolsep}{3pt}
  \small
  \renewcommand{\arraystretch}{1.02}
  \setlength{\tabcolsep}{2.3mm}

  \begin{tabular}{l|cc|cc|c|ccc}
    \toprule
    \multirow{2}{*}{Methods}
      & \multicolumn{2}{c|}{Func-Group}
      & \multicolumn{2}{c|}{Scaffold}
      & \multicolumn{1}{c|}{SMILES}
      & \multicolumn{3}{c}{Molecule-Edit} \\
    \cmidrule(r){2-3} \cmidrule(r){4-5} \cmidrule(r){6-6} \cmidrule(r){7-9}
      & FG$\downarrow$
      & Ring$\downarrow$
      & Murcko$\downarrow$
      & Ring-sys$\downarrow$
      & Eq.$\downarrow$
      & Add$\downarrow$
      & Delete$\downarrow$
      & Sub$\downarrow$ \\
    \midrule
    Single LLM
      & \textbf{129.30}
      & 194.47
      & 252.93
      & \textbf{165.20}
      & \textbf{203.19}
      & 224.47
      & 247.57
      & \textbf{184.17} \\
    Diver. Decode 
      & 369.77
      & \textbf{73.72}
      & \textbf{186.29}
      & 210.11
      & 325.33
      & \textbf{117.17}
      & \textbf{109.04}
      & 282.08 \\
    \midrule
    Single / DD
      & 0.35$\times$
      & \textbf{2.64$\times$}
      & \textbf{1.36$\times$}
      & 0.79$\times$
      & 0.62$\times$
      & \textbf{1.92$\times$}
      & \textbf{2.27$\times$}
      & 0.65$\times$ \\
    \bottomrule
  \end{tabular}
\end{table*}

We report the default checkpoint precision of all HuggingFace models used in our experiments in Table~\ref{tab:model_precision}. Unless otherwise specified, the precision corresponds to the tensor type or \texttt{torch\_dtype} reported by the corresponding HuggingFace model repository.
\begin{table}[!ht]
  \centering\
  \caption{Default checkpoint precision of HuggingFace models used in our experiments.}
  \label{tab:model_precision}
  \small
  \renewcommand{\arraystretch}{1.05}
  \setlength{\tabcolsep}{4pt}
  \begin{tabular}{l|l|l}
    \toprule
    Model & HuggingFace repository & Default precision \\
    \midrule
    R1-Distill-Qwen-32B 
      & \texttt{deepseek-ai/DeepSeek-R1-Distill-Qwen-32B} 
      & BF16 \\
    R1-Distill-LLaMA-70B 
      & \texttt{deepseek-ai/DeepSeek-R1-Distill-Llama-70B} 
      & BF16 \\
    ChemDFM-R 
      & \texttt{OpenDFM/ChemDFM-R-14B} 
      & BF16 \\
    Chem-R 
      & \texttt{weidawang/Chem-R-8B} 
      & BF16 \\
    TxGemma 
      & \texttt{google/txgemma-2b-predict} 
      & FP32 \\
    Raman-1.7B 
      & \texttt{think-a-tron/raman-01-1.7B} 
      & FP32 \\
    \bottomrule
  \end{tabular}
\end{table}
In addition to model precision, all the benchmarks and evaluation protocols used in our experiments are existing public evaluation datasets. ChemCoTBench is used to evaluate molecule understanding and editing abilities, ChemBench is used to evaluate broader chemistry knowledge across multiple subdomains, and GPQA-diamond is used to evaluate expert-level scientific reasoning in chemistry, biology, and physics.
\section{Case study}
\label{app-casestudy}
We present more cases to further demonstrate that divergence decoding can correct the reasoning process at an early stage, thereby helping the model focus more directly on the chemical facts relevant to the question.
\begin{figure}[h]
    \centering
    \includegraphics[width=1.\linewidth]{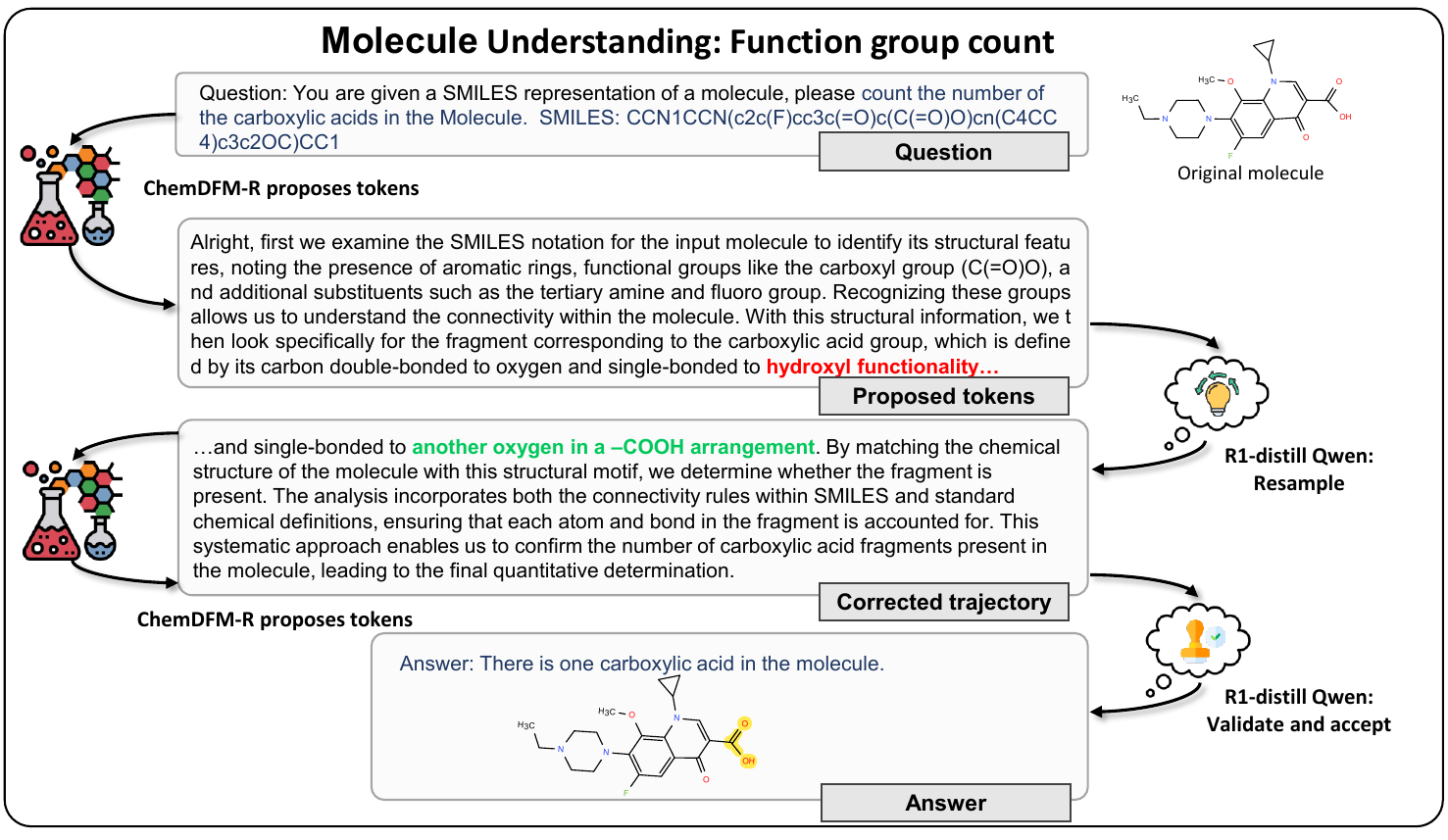}
    \captionsetup{skip=2pt}
    \caption{\textbf{\textit{Case study:}} divergence decoding between ChemDFM and R1-Distill Qwen affects the reasoning process in the function group counting task.}
    \label{fig:resample-study-app2}
\end{figure}

\begin{figure}[h]
    \centering
    \includegraphics[width=1.\linewidth]{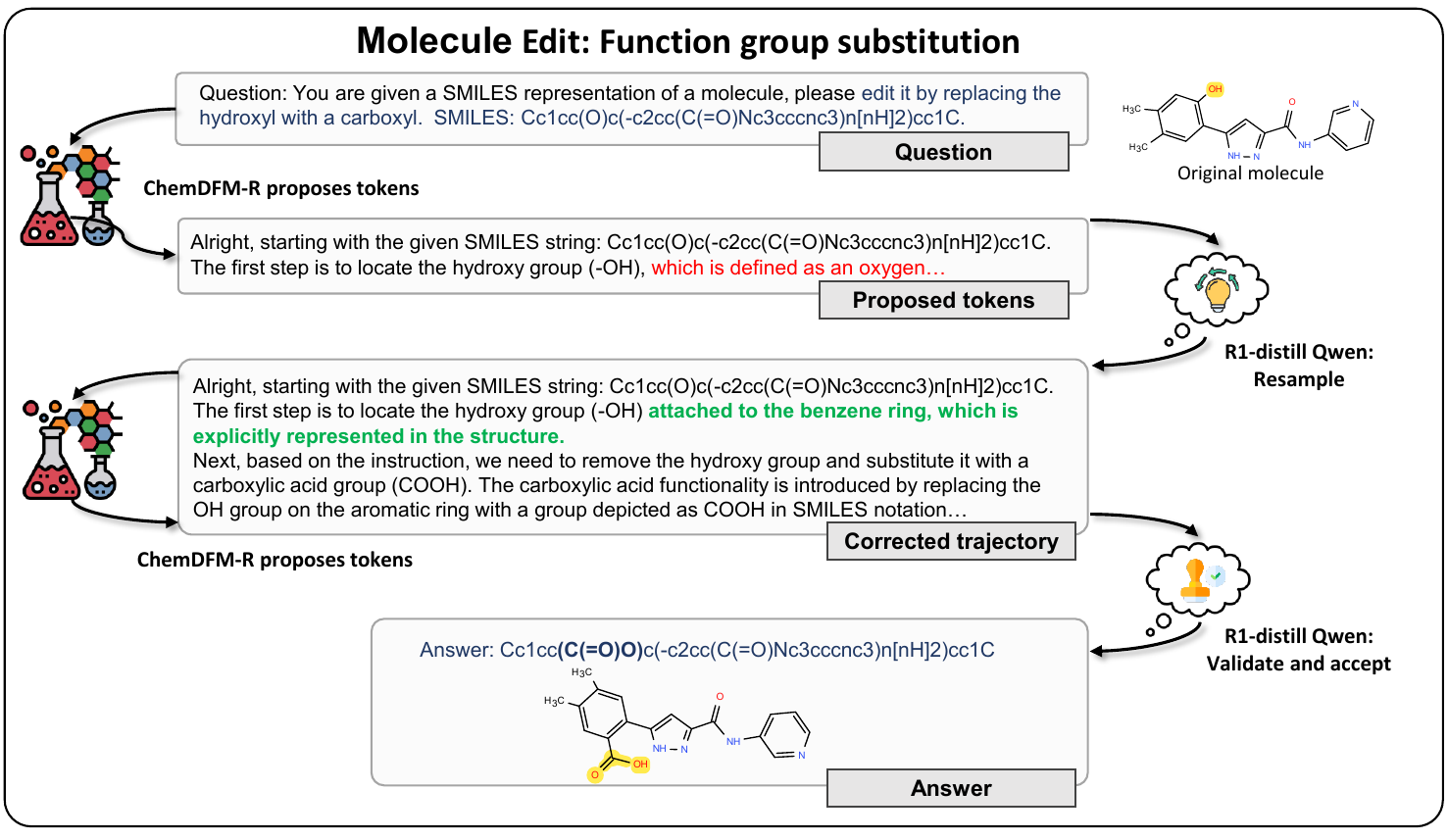}
    \captionsetup{skip=2pt}
    \caption{\textbf{\textit{Case study:}} divergence decoding between ChemDFM and R1-Distill Qwen affects the reasoning process in the molecule editing task.}
    \label{fig:resample-study-app}
\end{figure}

\begin{figure}[h]
    \centering
    \includegraphics[width=1.\linewidth]{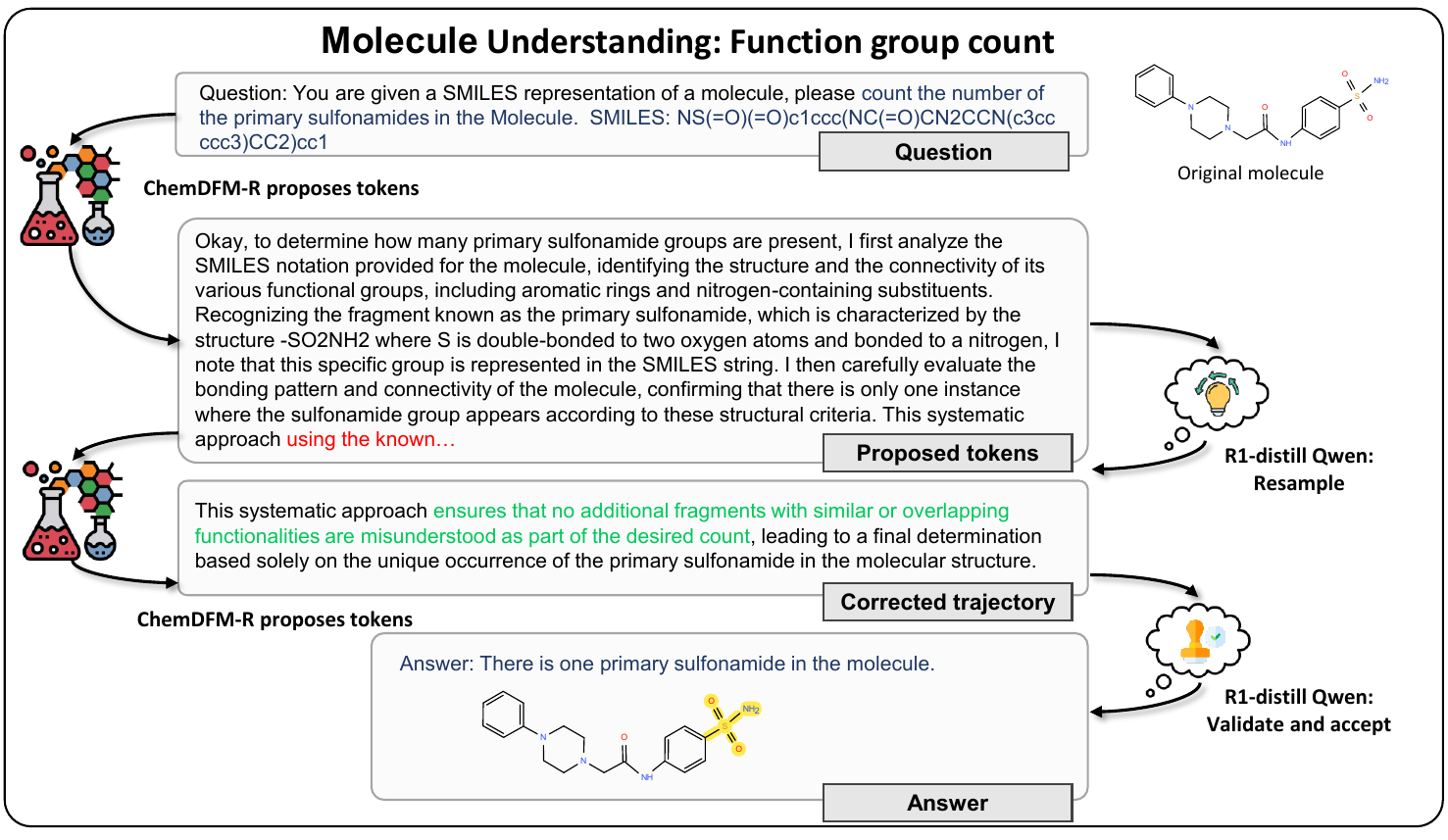}
    \captionsetup{skip=2pt}
    \caption{\textbf{\textit{Case study:}} divergence decoding between ChemDFM and R1-Distill Qwen affects the reasoning process in the function group counting task.}
    \label{fig:resample-study-app3}
\end{figure}

\begin{figure}[h]
    \centering
    \includegraphics[width=1.\linewidth]{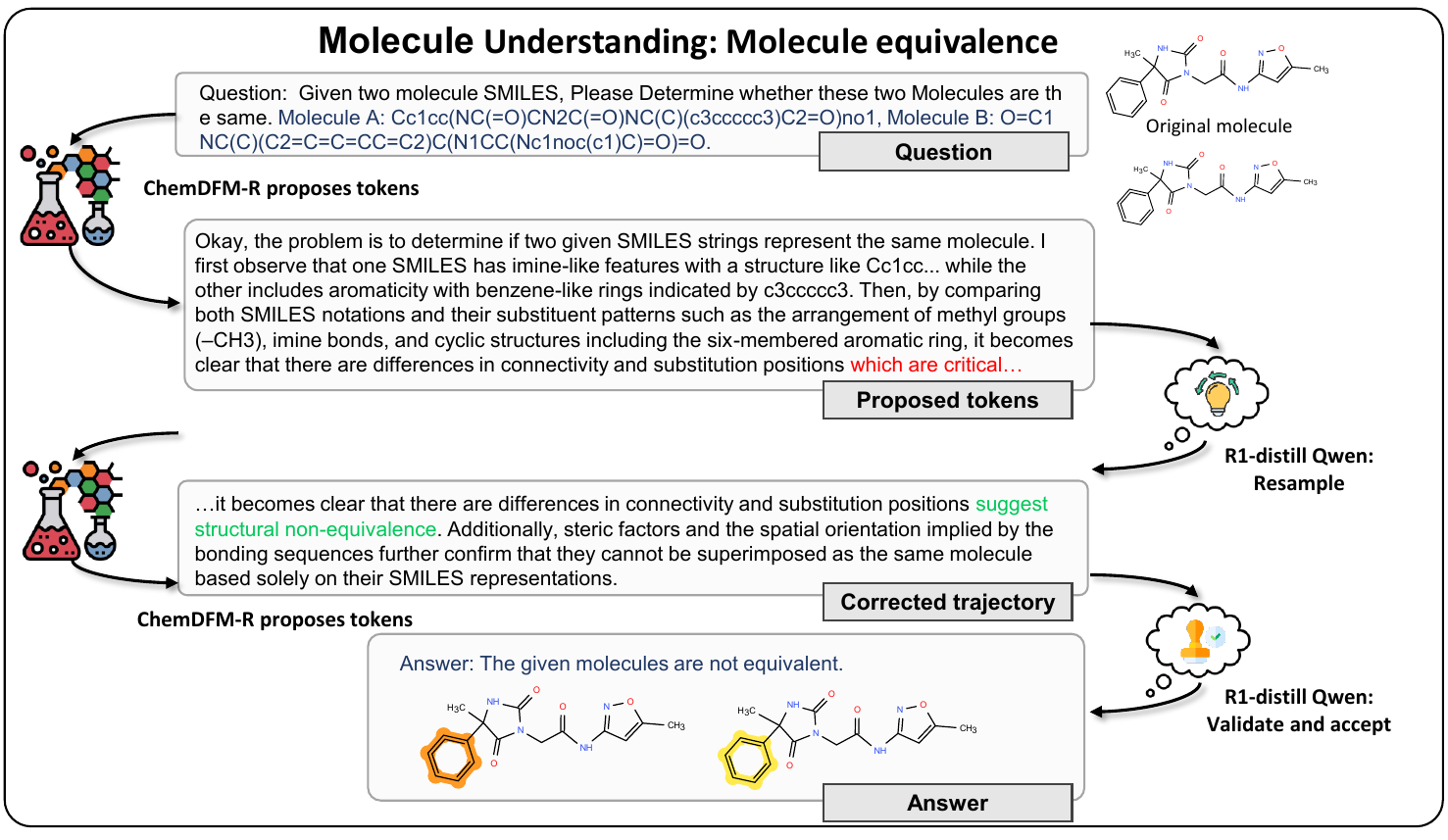}
    \captionsetup{skip=2pt}
    \caption{\textbf{\textit{Case study:}} divergence decoding between ChemDFM and R1-Distill Qwen affects the reasoning process in the function group counting task.}
    \label{fig:resample-study-app4}
\end{figure}

\section{Code Availability}
\label{anoycode}
The anonymous source code for reproducing our experiments is available at:
\url{https://github.com/wyattxuanyang/Divergence-Decoding}
\section{Limitations and Broader Impacts}
\paragraph{Limitations.}
Although Divergence Decoding provides a training-free way to fuse domain expertise and general reasoning, it still has several limitations. First, the method relies on the complementarity between the specialist model and the general reasoning model. If the two models make highly correlated errors, or if the generalist does not provide a reliable fallback signal, the benefit of token-level routing may be limited. Second, the method introduces additional hyperparameters, including the JS threshold, block size, temperature, and the top-$k$ support used for divergence computation. While our experiments show stable improvements under the chosen settings, different model pairs or domains may require further calibration. Third, although the draft-then-verify design can offset part of the overhead, Divergence Decoding still requires loading and querying two LLMs, leading to higher memory requirements than single-model decoding.

\paragraph{Broader Impacts.}
Divergence Decoding may have positive impacts by enabling stronger scientific reasoning systems without additional training or fine-tuning. Since it composes existing specialist and generalist LLMs at inference time, it can reduce the cost of building domain-specific reasoning systems and may support scientific tasks such as chemistry understanding, molecular editing, and expert-level scientific question answering. More broadly, the method provides a practical framework for reusing heterogeneous LLM capabilities and improving reliability through adaptive collaboration.

At the same time, stronger scientific LLM systems may also introduce risks. Users may over-rely on generated reasoning traces or final answers, especially in high-stakes scientific, biomedical, or engineering settings where incorrect conclusions can have real-world consequences. In addition, improved scientific reasoning ability should not be interpreted as a substitute for expert validation. We therefore recommend using Divergence Decoding as an assistive tool rather than an autonomous decision maker, and outputs should be verified by domain experts before being used in safety-critical applications. Future work should further evaluate the robustness, uncertainty calibration, and safety behavior of divergence-based model collaboration before deployment in real-world scientific workflows.




\end{document}